%% file: main.tex
\documentclass{article}
\usepackage{preprint,times}
\usepackage{amsmath,amssymb,amsfonts,amsthm}
\usepackage{bm}
\usepackage{booktabs}
\usepackage{array}
\usepackage{tabularx}
\usepackage{graphicx}
\usepackage{xcolor}
\usepackage{algorithm}
\usepackage{algpseudocode}
\usepackage{microtype}
\usepackage{float}
\usepackage{wrapfig}
\usepackage{capt-of}
\usepackage{needspace}
\usepackage[section]{placeins}
\usepackage{tikz}
\usetikzlibrary{arrows.meta,positioning,calc,fit,shapes.geometric}
\usepackage{url}
\usepackage{hyperref}

\newcommand{\E}{\mathbb{E}}
\newcommand{\Cov}{\operatorname{Cov}}

\newcommand{\tr}{\operatorname{tr}}
\newcommand{\vecop}{\operatorname{vec}}
\newcommand{\norm}[1]{\left\lVert #1\right\rVert}
\newcommand{\loss}{\mathcal{L}}

\newcolumntype{Y}{>{\raggedright\arraybackslash}X}
\newtheorem{proposition}{Proposition}

\usepackage{mathtools}
\newcommand{\R}{\mathbb R}
\DeclareMathOperator{\sign}{sign}
\DeclareMathOperator{\polar}{polar}

\providecommand{\doi}[1]{\href{https://doi.org/#1}{doi:\nolinkurl{#1}}}
\allowdisplaybreaks[1]

\AddToHook{env/thebibliography/begin}{\interlinepenalty=10000}

\hypersetup{hidelinks,
 pdftitle={How Inner Optimizers Shape Replica Scaling},
 pdfauthor={Eugene Belilovsky, Stefan Horoi, Edouard Oyallon, Benjamin Therien}}

\title{How Inner Optimizers Shape Replica Scaling}

\author{%
{\small Eugene Belilovsky$^{1,2}$\quad
Stefan Horoi$^{1,3}$\quad
Edouard Oyallon$^{4}$\quad
Benjamin Th\'erien$^{1,3}$}\\[3pt]
{\normalfont\small $^{1}$Mila\quad
$^{2}$Concordia University\quad
$^{3}$Universit\'e de Montr\'eal\quad
$^{4}$CNRS}
}

\begin{document}
\maketitle
{\renewcommand{\thefootnote}{}%

 \addtocounter{footnote}{-1}}

\vspace{-10pt}
\begin{abstract}
Training large language models (LLMs) across geographically distributed and poorly interconnected compute pools or data centers has become a topic of significant interest. This has led to increased study of local optimization methods such as DiLoco for LLMs, especially those that use local (or inner) adaptive or preconditioned optimizers, which are critical for high-performance LLM training. However, a central challenge remains: it is well known that increasing the number of model replicas (replica scaling) while fixing the total training budget leads to monotonically decreasing performance. Existing work has empirically demonstrated that the choice of inner optimizer (Muon v.s. AdamW) can significantly improve this trend, but offers only limited explanation of the performance difference. In this work, we consider a range of inner-optimizers and identify two important factors that impact replica scaling: (1) the covariance structure of the pseudo-gradients and (2) the impact of batch size on the pseudogradient's noise scale.
Based on these insights, we theoretically and empirically demonstrate that whitening-based pre-conditioning is a powerful inner-optimizer desideratum. Our empirical evaluation extends the results to reveal that whitening-based inner-optimizers achieve strong replica scaling (for $8-96$ replicas), with performance relative to DP improving as parameters and data are scaled.

\end{abstract}
\vspace{-6pt}

\section{Introduction}
\noindent\begin{minipage}[t]{.565\linewidth}
\vspace{0pt}
The computational scale of language-model training increasingly exceeds the
capacity of a single site. Training across geographically separated clusters
can pool otherwise fragmented resources, but the connectivity between these clusters is poor, leading to significant communication bottlenecks. Local optimization methods address this
constraint by allowing replicas to optimize independently for several steps
between parameter or pseudo-gradient averages
\citep{douillard2023diloco,douillard2025streaming,liu2024async} often drastically reducing the frequency of communication.

A central difficulty in practical local optimization is replica scaling: at a fixed token budget, increasing the number of replicas often degrades the final model performance
\citep{charles2025scaling} relative to its non-local data-parallel (DP) counterpart. This can limit the broad applicability of these methods. It has been shown that the replica scaling can improve with increasing model size \citep{charles2025scaling}. Recent empirical results show that modifying either the inner optimizer
\citep{therien2025muloco} or the outer aggregation rule
\citep{horoi2026isoloco} can further improve the scalability and efficiency of distributed training. However, the impact of inner optimizers on scaling is not well understood. More broadly, the primary sources of divergence from standard non-local training are not well understood.
\end{minipage}\hfill
\begin{minipage}[t]{.41\linewidth}
\vspace{0pt}
\centering
\includegraphics[width=\linewidth]{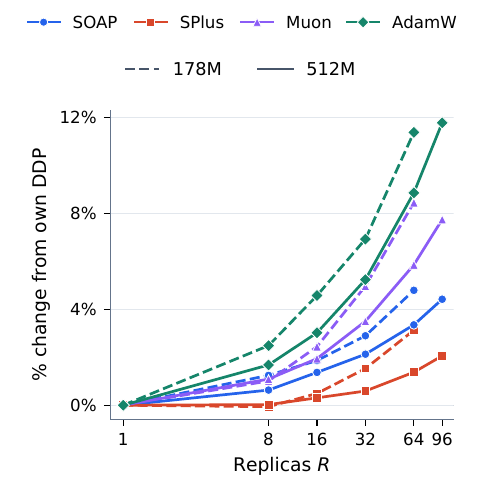}
\setlength{\abovecaptionskip}{1pt}
\vspace{-20pt}
\captionof{figure}{\small \textbf{Local optimization performance as replicas are scaled.} Each local optimizer is compared to its own DP baseline, to eliminate confounding replica scaling with absolute loss. Optimizers that whiten the pseudogradient covariance achieve the strongest worker scaling. SPlus is the best proxy for whitening. We note that SOAP's stronger absolute performance sometimes leads to the best results despite its poorer worker scaling. Notably, we also observe a strong scaling trend relative to DP for all inner-optimizers.}
\label{fig:scaling-178m-vs-512m}
\end{minipage}
\par\Needspace{5\baselineskip}

Classical analyses of local SGD and federated averaging study the discrepancy accumulated over multiple local steps, often with heterogeneous client data. The primary source of performance degradation is often attributed to client drift or other notions of divergence of model replicas along their trajectory \citep{karimireddy2020scaffold}. Much of this work and theory \citep{stich2019local,karimireddy2020scaffold,yu2019momentum} focuses on SGD and momentum SGD as the inner optimizer, a single local step from common parameters and optimizer states that commute with averaging: averaging the updates gives the same result as updating with the averaged gradient. Nonlinear adaptive inner optimizers need not satisfy this property and are much less well understood in the context of this literature. Notably, a discrepancy can arise on the first step, even with identical starting states and identically distributed data, because transforming each local gradient before averaging differs from transforming the averaged gradient. Understanding replica scaling therefore requires examining how the inner optimizer transforms gradient noise, in addition to how parameters drift over successive local steps.

In this work, we focus on a broad range of popular
\emph{adaptive} or \emph{preconditioner} based optimizers and identify two critical properties that can greatly impact scaling: (a) the covariance structure of the pseudo-gradient induced by the local optimization (b) how the optimizers deal with statistical estimation of second moments in the low batch size setting.

Our contributions can be summarized as follows: 
\begin{itemize}
    \item We demonstrate, both empirically and theoretically, the importance of the pseudogradient covariance structure induced by the choice of inner optimizer. Specifically, we link the choice of inner optimizer to the maximum discrepancy with respect to the DP update direction. 
    \item We identify the importance of appropriately scaling $\beta_2$ when using second-moment optimizer states, validating our findings through theoretical analysis and experiments. This reveals that the recently introduced IsoLoCo method~\citep{horoi2026isoloco}, using an Adam inner optimizer, can outperform the much more computationally expensive MuLoCo~\citep{therien2025muloco}.
    \item Our results demonstrate that using a SOAP~\citep{vyas2025soap} or SPlus~\citep{frans2025splus} inner optimizer can lead to very effective replica scaling , obtaining significantly improved performance relative to existing methods from the literature (as much as $6\%$ less degradation relative to DP at 96 replicas). %
    
\end{itemize}

\section{Related Work}

\paragraph{Local-step methods and DiLoCo.} Local SGD \citep{stich2019local} and
federated averaging \citep{mcmahan2017fedavg} reduce communication by taking
several local steps between averaging events. DiLoCo
\citep{douillard2023diloco} adapts this to LLM pretraining by utilizing a large number of inner steps, AdamW \citep{loshchilov2019adamw} as the inner optimizer, and SGD with Nesterov momentum on the aggregated pseudo-gradients as the outer optimizer. Follow-ups relax synchrony
\citep{liu2024async} and overlap or shard the exchange of information between replicas
\citep{douillard2025streaming}. \citet{charles2025scaling} establish scaling
laws for DiLoCo and document the central difficulty we study: adding replicas
at a fixed token budget degrades quality, with the optimal hyperparameters
shifting as $R$ grows. Our analysis complements these scaling laws by tracing
the degradation back to local state estimation and to the nonlinear optimizer
transforms applied before aggregation. Local methods can also be complemented with local optimization by using Sparse parameter averaging
\citep[SPARTA;][]{reed2025sparta} or Top-$k$ sparsification SparseLoCo \citep{sarfi2025sparseloco}. IsoLoco \cite{horoi2026isoloco} views DiLoCo's outer step as task-arithmetic model merging and replaces plain pseudo-gradient averaging with Iso-C merging~\citep{marczak2025isoc}. For each weight matrix, it takes the SVD of the averaged pseudo-gradient and replaces all singular values with their mean, giving $\bar{\sigma}UV^\top$. It then applies the usual Nesterov momentum outer step. This was shown to improve worker scaling over DiLoCo~\citep{horoi2026isoloco}. 

\paragraph{Small-batch training and adaptive-optimizer scaling rules.} Scaling
rules relate optimizer hyperparameters to batch size
\citep{malladi2022sde}; most directly, \citet{marek2025small} show that Adam
trains stably at very small batches provided $\beta_2$ is set so that the second
moment's half-life is fixed in \emph{tokens} rather than in steps. Replica
scaling in DiLoCo induces exactly this small-batch regime for each individual replica.
Section~\ref{rw:sec:h1} transfers the $\beta_2$ scaling rule to replica scaling; to our knowledge this hasn't been studied has not been studied in the local-step setting previously.

\paragraph{Spectral and second-order optimizers.} Shampoo
\citep{gupta2018shampoo} preconditions with Kronecker-factored curvature
estimates; SOAP \citep{vyas2025soap} runs Adam in Shampoo's eigenbasis; Muon
\citep{jordan2024muon} orthogonalizes matrix updates with no accumulated
second-moment statistics. SPlus \citep{frans2025splus} retains Shampoo's
historical eigenbasis but signs momentum in that basis, avoiding cached
eigenvalue magnitudes and SOAP's additional coordinate-wise second moment. SPlus
uses this instant-sign transformation independently in local DiLoCo replica
bases; our focus is its distributed covariance and replica-scaling behavior.
Controlled single-stream comparisons find the
ranking horizon-dependent \citep{wen2025fantastic}, while a recent
multi-billion-parameter study reports a slight edge for a stabilized SOAP over
Muon at very large batch, at the price of memory and numerical-stability
engineering \citep{khona2026soapmuon}. In the DiLoCo setting, MuLoCo
\citep{therien2025muloco} uses Muon as the inner optimizer and reports more
directionally consistent pseudo-gradients as $R$ grows, while IsoLoCo
\citep{horoi2026isoloco} treats outer aggregation as model merging and
isotropizes the aggregated pseudo-gradient. Section~\ref{rw:sec:theory}
develops a complementary criterion: whether the stochastic update geometry
emitted by each replica remains well conditioned under averaging.

\section{Background: the DiLoCo Framework}
\label{sec:background}

The methods in this paper can be seen as an instance of Algorithm~\ref{alg:diloco} which is a generalization of local optimization from \cite{reddi2021fedopt,douillard2023diloco} that utilizes both inner (on the local gradients) and outer optimization (on the pseudo-gradients). $R$
replicas start each round from the shared parameters, each runs $H$ steps of an
inner optimizer on its own data stream. The resulting displacements
$\bm{\Delta}_r$ are uniformly averaged, and an outer optimizer applies the aggregate $\bar{\bm{\Delta}}$ to the
shared parameters from the start of the round. DiLoCo \citep{douillard2023diloco} instantiates the inner
optimizer as AdamW and the outer as Nesterov momentum. For example MuLoCo
\citep{therien2025muloco} replaces the inner optimizer with Muon. IsoLoCo \citep{horoi2026isoloco} isotropizes
$\bar{\bm{\Delta}}$ before the outer step. The inner optimizer states $s_r$ are never
communicated: each is estimated locally from its replica's data stream.

\begin{algorithm}[H]
\caption{The DiLoCo framework. All methods studied here are instances of this algorithm.}
\label{alg:diloco}
\begin{algorithmic}[1]
\Require initial parameters $\bm{\theta}^{(0)}$; replicas $r\in\{1,\dots,R\}$; inner steps $H$;
inner optimizer $\textsc{Inner}$ with per-replica state $s_r$;
outer optimizer $\textsc{Outer}$ with state $S$; global batch size $B$;
\For{outer round $t=0,\dots,T-1$}
  \ForAll{replicas $r$ \textbf{in parallel}}
    \State $\bm{\theta}_r \gets \bm{\theta}^{(t)}$
    \For{$h=1,\dots,H$}
      \State $\bm{g} \gets \nabla_{\bm{\theta}}\loss(\bm{\theta}_r;x_{r,h})$
        \Comment{local minibatch: $B/R$ tokens}
      \State $(\bm{\theta}_r,\,s_r) \gets \textsc{Inner}(\bm{\theta}_r,\,\bm{g},\,s_r)$
        \Comment{$s_r$ is never synchronized}
    \EndFor
    \State $\bm{\Delta}_r \gets \bm{\theta}^{(t)}-\bm{\theta}_r$
  \EndFor
  \State $\bar{\bm{\Delta}} \gets \tfrac1R\sum_r\bm{\Delta}_r$ \Comment{uniform average}
  \State $(\bm{\theta}^{(t+1)},\,S) \gets \textsc{Outer}(\bm{\theta}^{(t)},\,\bar{\bm{\Delta}},\,S)$
    \Comment{Nesterov momentum in DiLoCo}
\EndFor
\end{algorithmic}
\end{algorithm}

\section{Inner Optimizers Impact on Replica Scaling}
\label{rw:sec:theory}

The transformations applied to local updates by local optimizers can have a significant impact on the final aggregated pseudo-gradient. We investigate how the choice of inner optimizer shapes the distribution of the resulting pseudo-gradients and their deviation from the corresponding non-local updates.  Section~\ref{rw:sec:maps} introduces the optimizers considered in this work, 
Section~\ref{rw:sec:decomposition} separates the update error (relative to non-local steps with the same base optimizer) into
a \textbf{mean mismatch} term and a \textbf{stochastic error} term. It shows how whitening controls the largest directional variance, and how matching covariance structures across replicas makes uniform averaging optimal.
Section~\ref{rw:sec:h1} specializes this framework to one step from common
state: under the stated Gaussian model, SPlus admits a scalar
calibration that matches the DP mean to first order in the global proposal's
signal-to-noise ratio.
Finally, Section~\ref{rw:sec:h1} bounds the Adam mean mismatch caused by
fresh-gradient second moments and explains how appropriately adjusting $\beta_2$ limits that
bound as local batches shrink. It then isolates a complementary advantage
of SPlus: retaining the estimated basis without relying on coordinatewise
scale estimates to control the covariance term. 

\subsection{Background on Popular Optimizers}
\label{rw:sec:maps}
In this section we present a number of popular optimizers using a unified notation to facilitate their common analysis and comparison. In particular, we include recent matrix-based optimizers such as Muon and those in the Shampoo family. The optimizers differ in how they retain or remove the magnitudes of a
momentum proposal before averaging across replicas. We write these
transformations explicitly so their mean and covariance can be studied in
the same framework.

Consider a matrix parameter with stochastic gradient $\bm{G}\in\R^{m\times n}$. An optimizer returns a direction
$\bm{U}$ and an updated state after observing $\bm{G}$; its parameter displacement is
$\bm{\Delta}=\eta\,\vecop(\bm{U})$. We use subscripts $t-1$ and $t$ for
the incoming and updated optimizer states, respectively, and suppress the step
index on the gradient and returned direction. The first-moment proposal is
\begin{equation}
 \bm{M}_t=\beta_1\bm{M}_{t-1}+(1-\beta_1)\bm{G}.
 \label{rw:eq:momentum}
\end{equation}
We compare the following matrix-update cores, using standard EMA momentum to
isolate their normalizations. For simplicity we omit standard deterministic
bias corrections and weight decay from the presentation. The matrix-size scale $\gamma$ of Muon is shown
explicitly. All decay coefficients lie in $[0,1)$; numerical damping is denoted by $\epsilon>0$.

\paragraph{Adam and Shampoo.}
Adam uses an entrywise second moment in parameter coordinates
\citep{kingma2015adam,loshchilov2019adamw}:
\begin{equation}
 \bm{V}_{A,t}=\beta_2\bm{V}_{A,t-1}+(1-\beta_2)\bm{G}^{\odot2},
 \qquad \bm{U}_{\rm Adam}=\bm{M}_t\oslash(\sqrt{\bm{V}_{A,t}}+\epsilon).
 \label{rw:eq:adam-map}
\end{equation}
Where $U$ is the update. Here, $\odot$ and $\oslash$ denote elementwise multiplication
and division, respectively, with $\bm{G}^{\odot 2}$
representing the elementwise square. Shampoo instead maintains left and right Gram factors~\citep{gupta2018shampoo}:
\begin{equation}
 \begin{split}
 \bm{L}_t&=\beta_3\bm{L}_{t-1}+(1-\beta_3)\bm{G}\bm{G}^\top,\qquad
 \bm{\mathsf R}_t=\beta_3\bm{\mathsf R}_{t-1}+(1-\beta_3)\bm{G}^\top \bm{G},\\
 \bm{U}_{\rm Shampoo}&=(\bm{L}_t+\epsilon \bm{I})^{-1/4}\bm{M}_t
                         (\bm{\mathsf R}_t+\epsilon \bm{I})^{-1/4}.
 \end{split}
 \label{rw:eq:shampoo-map}
\end{equation}
The above formulation is the canonical version of Shampoo. We note that some variants consider half-power instead of quarter~\citep{vyas2025soap}. We also note variants
with grafting or additional normalizations are often used, which can impact the interpretation of the underlying preconditioner~\citep{lin2026kl}.

\paragraph{SOAP and SPlus.}
Let $\bm{Q}_L,\bm{Q}_R$ be orthogonal eigenvector matrices obtained from the historical
Gram factors above. Their action is the matrix rotation
\begin{equation}
 \widetilde{\bm{G}}=\bm{Q}_L^\top \bm{G}\bm{Q}_R,\qquad
 \widetilde{\bm{M}}_t=\bm{Q}_L^\top \bm{M}_t\bm{Q}_R.
 \label{rw:eq:rotation}
\end{equation}
SOAP (ShampoO with Adam in the Preconditioner's eigenbasis) performs Adam-style normalization in this product basis \citep{vyas2025soap}:
\begin{equation}
 \begin{split}
 \bm{V}_{S,t}&=\beta_2\bm{V}_{S,t-1}+(1-\beta_2)\widetilde{\bm{G}}^{\odot2},\\
 \bm{U}_{\rm SOAP}&=\bm{Q}_L\big[\widetilde{\bm{M}}_t\oslash
                                  (\sqrt{\bm{V}_{S,t}}+\epsilon)\big]\bm{Q}_R^\top.
 \end{split}
 \label{rw:eq:soap-map}
\end{equation}
In this work, we consider a signed update in a Shampoo basis, a variant of the recently introduced SPlus \citep{frans2025splus} optimizer, which does not utilize iterate averaging. This optimizer retains the momentum and historical basis but signs each projected entry:
\begin{equation}
 \bm{U}_{\rm SPlus}=\bm{Q}_L\sign(\widetilde{\bm{M}}_t)\bm{Q}_R^\top,
 \qquad \sign(0)=0.
 \label{rw:eq:suds-map}
\end{equation}
This approach can be seen as a simplified version of SOAP that removes the dependence on the second moment. In the following sections, we will further illustrate how this optimizer naturally arises in the context of local optimization with replica scaling.
SOAP's rotated normalization and SPlus sign transform have a whitening
interpretation \citep{vyas2025soap,frans2025splus,lin2026kl}.
Section~\ref{rw:sec:decomposition} gives the corresponding covariance
identities under explicit assumptions and explains their impact on pseudo-gradient aggregation.

\paragraph{Muon.}
For a  singular value decomposition $\bm{M}_t=\bm{U}_M\bm{\Sigma}_M \bm{V}_M^\top$ on the
nonzero singular subspace, ideal Muon returns
\begin{equation}
 \bm{U}_{\rm Muon}=\gamma\polar(\bm{M}_t)=\gamma \bm{U}_M\bm{V}_M^\top.
 \label{rw:eq:muon-map}
\end{equation}
It replaces nonzero singular values by a common value
\citep{jordan2024muon}. In matched singular-vector coordinates,
$\bm{U}_M^\top \bm{M}_t\bm{V}_M=\bm{\Sigma}_M$, so at its core Muon also has the form
``rotate, remove magnitudes, rotate back.'' Unlike SPlus, its basis is determined
by the current proposal, not by historical covariance statistics. In practice, Newton-Schulz is used to approximate the polar map.

\subsection{Impact of Covariance Structure on Local Optimization}
\label{rw:sec:decomposition}

We consider the general local optimization setting (under any local optimizer) and separate the error relative to the expected DP update with the same optimizer into two terms: (1) \textbf{mean mismatch} and (2) \textbf{stochastic error}.
The covariance analysis in this subsection addresses the stochastic term;
Section~\ref{rw:sec:h1} then addresses the mean mismatch. Following the pseudo-gradient interpretation of client displacements in
\citet{reddi2021fedopt} and \citet{wang2024unreasonable}, we treat each replica's
complete local update as a random vector. All expectations and covariances below are conditional on the start-of-round parameters and the incoming optimizer states. Let $\bm{\Delta}_r\in\R^d$ be the complete $H$-step displacement returned
by replica $r$. We decompose each local displacement into its conditional mean and a zero-mean stochastic component:
\begin{equation}
 \bm{\Delta}_r=\bm{\mu}_r+\bm{\varepsilon}_r,\qquad
 \bm{\mu}_r=\E[\bm{\Delta}_r],\qquad \E[\bm{\varepsilon}_r]=0,\qquad
 \bm{C}_r=\Cov(\bm{\Delta}_r).
 \label{rw:eq:replica-model}
\end{equation}
Assume finite second moments, $\mathbb{E}[\|\Delta_r\|_2^2]<\infty$ for all $r$, and zero cross-replica noise covariances, $\operatorname{Cov}(\varepsilon_r,\varepsilon_s)=\bm{0}$ for all $r\neq s$.
Independent local data streams, conditional on the start-of-round parameters and optimizer states, are sufficient to ensure zero cross-replica noise covariances.
Set $\bar{\bm{\Delta}}=R^{-1}\sum_r\bm{\Delta}_r$ and $\bar{\bm{\mu}}=R^{-1}\sum_r\bm{\mu}_r$. Note that for the typical setting studied in \cite{douillard2023diloco,therien2025muloco} with i.i.d. data and identical incoming optimizer states, $\bm{\mu}_r$ is the same for all replicas. 

We now consider $\bm{d}_H^{\rm DP}$: the displacement of the \emph{same inner
optimizer} run for $H$ \emph{synchronized} steps and same global batch size, in expectation. We allow a deterministic scalar calibration factor $\alpha_R$, and define
\begin{equation}
 \bm{b}_R=\alpha_R\bar{\bm{\mu}}-\bm{d}_H^{\rm DP},\qquad
 \bm{C}_{\rm agg}=\Cov(\bar{\bm{\Delta}})=\frac1{R^2}\sum_r \bm{C}_r.
 \label{rw:eq:bias-and-cov}
\end{equation}
For $H=1$, $\alpha_R$ can represent a step-size change, for example. We now consider the following decomposition of the error in any direction and the corresponding bound: 

\begin{proposition}[Directional error decomposition]
\label{rw:prop:decomposition}
The conditional bias--variance decomposition gives, for every unit vector $\bm{q}$,
\begin{equation}
 \underbrace{\E\!\left[\big(\bm{q}^\top(\alpha_R\bar{\bm{\Delta}}-\bm{d}_H^{\rm DP})\big)^2\right]}
             _{\text{error relative to the DP mean}}
 =\underbrace{(\bm{q}^\top \bm{b}_R)^2}_{\text{mean mismatch}}
  +\underbrace{\alpha_R^2\bm{q}^\top \bm{C}_{\rm agg}\bm{q}}_{\text{stochastic error}}.
 \label{rw:eq:main-decomposition}
\end{equation}
Consequently, the largest directional error is
\begin{equation}
 \lambda_{\max}(\bm{b}_R\bm{b}_R^\top+\alpha_R^2\bm{C}_{\rm agg})
 \leq \|\bm{b}_R\|_2^2+\alpha_R^2\lambda_{\max}(\bm{C}_{\rm agg}).
 \label{rw:eq:worst-risk}
\end{equation}
\end{proposition}
\begin{proof}
Subtract and add $\alpha_R\bar{\bm{\mu}}$. Using the Rayleigh quotient gives
\eqref{rw:eq:worst-risk}.
\end{proof}

The second term allows us to analyze the effect of inner optimizers that induce whitened updates. %
As we show below, whitening helps control the stochastic error by bounding the largest directional variance of the aggregated pseudo-gradient, while also promoting covariance matching across replicas.
The mean mismatch term is more challenging to analyze; however in Section~\ref{rw:sec:h1} we derive explicit expressions and bounds for several inner optimizers under a single-local-step setting. Moreover we will show in the empirical results that this term is generally dominated in the largest noise direction.

\paragraph{Whitening and directional variance.}
A linear map $\bm{P}$ whitens a covariance $\bm{\Sigma}$ if $\bm{P}\bm{\Sigma} \bm{P}^\top=\bm{I}$
\citep{kessy2018whitening}; a prescribed scalar multiple gives a spherical
covariance at another scale. For any covariance $\bm{C}$ of dimension $d$ and trace $\tau$ (fixed total variance), the following inequality holds:
\begin{equation}
 \lambda_{\max}(\bm{C})\geq\tau/d,
 \quad\text{with equality precisely when }\bm{C}=(\tau/d)\bm{I}.
 \label{rw:eq:isotropic-minimax}
\end{equation}
Thus, isotropy minimizes the \emph{largest} directional variance
at fixed total variance, precisely the covariance quantity
appearing in \eqref{rw:eq:worst-risk}.
Moreover, a bound on the maximum directional variance of each
replica carries over to the aggregate:
\[
\lambda_{\max}(\bm{C}_{\rm agg})
\leq \frac{1}{R^2}\sum_r\lambda_{\max}(\bm{C}_r).
\]
Consequently, whitening local pseudo-gradients toward a common
isotropic covariance controls the stochastic error.

\paragraph{Whitening and optimality of uniform averaging.} Local update whitening has an additional, pre-aggregation benefit when inner
optimizer trajectories produce different update distributions. For the stochastic term
alone, a weighted aggregate has directional variance
\[
\sum_{r=1}^R w_r^2\,\bm q^\top\bm C_r\bm q,
\qquad \sum_r w_r=1.
\]
Uniform weights, which are used by default in standard DiLoCo settings, minimize this quantity when the replicas have equal
directional variances~\citep{aitken1935least}; they do so in every direction when
$\bm C_1=\cdots=\bm C_R$. Whitening each replica toward the same spherical
target $\nu\bm I$ promotes precisely this agreement, since
$
\|\bm C_i-\bm C_j\|
\leq
\|\bm C_i-\nu\bm I\|
+
\|\bm C_j-\nu\bm I\|.
$

Thus, local whitening both controls the worst-direction aggregate noise and makes the default uniform averaging well suited to the individual replica noise geometries. However, it does not control the mean mismatch $\bm b_R$, which is analyzed separately in Section~\ref{rw:sec:h1}.

\subsection{Impact of Optimizer Transformations on Update Covariance}
In this subsection, we analyze how the transformations induced by different inner optimizers affect the covariance structure of their updates under idealized assumptions, with particular emphasis on their whitening properties.
This comparison is motivated by the
eigenbasis interpretation of SOAP \citep{vyas2025soap}, the whitening
motivation of \cite{frans2025splus}, and the covariance-estimation analysis
of \cite{lin2026kl}.

For this comparison, consider a random matrix proposal
$\bm M = \bm M_t \in \mathbb R^{m\times n}$, with
$\bm{\Sigma} = \Cov(\vecop(\bm M)) \succ 0$.
Let $\bm U$ denote the resulting update matrix, and define
$d=mn$ and $\bm C=\Cov(\vecop(\bm U))$.
We treat optimizer bases and linear preconditioners as fixed independently of the sampled proposal, and define ideal normalization using the proposal's \emph{centered} covariance. Write
\[
 \bm{Q}=\bm{Q}_R\otimes \bm{Q}_L,
 \qquad \vecop(\bm{Q}_L^\top \bm{M}\bm{Q}_R)=\bm{Q}^\top\vecop(\bm{M}).
\]

We now examine how each optimizer transforms the covariance of the stochastic proposal and under which conditions the resulting update covariance is isotropic.

\paragraph{Adam.}
Suppose each entry of $\bm M$ is divided by its true standard deviation.
Writing $\bm D=\operatorname{diag}(\Sigma_{11},\ldots,\Sigma_{dd})$, this gives
\[
 \vecop(\bm U)=\bm D^{-1/2}\vecop(\bm M),\qquad
 \bm C=\bm D^{-1/2}\bm\Sigma\bm D^{-1/2}. 
\]
Every diagonal entry of $\bm C$ is one, but
$C_{ij}=\Sigma_{ij}/\sqrt{\Sigma_{ii}\Sigma_{jj}}$.
Thus, even with access to the true coordinate variances, Adam's normalization preserves the correlations present in $\bm M$.

\paragraph{Shampoo.}
Consider the idealized case where the factors equal the population Gram
matrices of a zero-mean input:
$\bm L=\mathbb E[\bm M\bm M^\top]$ and
$\bm{\mathsf R}=\mathbb E[\bm M^\top\bm M]$.
Assume the covariance is separable,
$\bm\Sigma=(\bm{\mathsf R}\otimes\bm L)/\operatorname{tr}(\bm L)$,
with positive-definite factors. The canonical Shampoo transformation then gives:
\[
 \bm U=\bm L^{-1/4}\bm M\bm{\mathsf R}^{-1/4},\qquad
 \bm C=\frac{\bm{\mathsf R}^{1/2}\otimes\bm L^{1/2}}
 {\operatorname{tr}(\bm L)}.
\]
The condition number of the resulting covariance, i.e., the ratio of its largest to smallest eigenvalues, is the square root of that of the input covariance. Thus, canonical Shampoo does not generally whiten the updates, as also discussed in \cite{lin2026kl}.

\paragraph{SOAP.}
Suppose the fixed rotation
$\widetilde{\bm M}_t=\bm Q_L^\top\bm M_t\bm Q_R$ gives
uncorrelated coordinates, and let $\bm V_{S,t}$ denote their true variances:
$(\bm V_{S,t})_{ij}=\operatorname{Var}((\widetilde{\bm M}_t)_{ij})$.
Dividing each rotated entry by its true standard deviation and rotating back gives
\[
 \bm U=\bm Q_L[\widetilde{\bm M}_t\oslash\sqrt{\bm V_{S,t}}]\bm Q_R^\top,
 \qquad \bm C=\bm I.
\]
The resulting variance is one in every unit direction, not only along the
original coordinates. Thus, SOAP can whiten the update covariance when the input covariance has the assumed Kronecker structure, the rotation $\bm Q$ decorrelates the proposal coordinates, and their variances $\bm V_{S,t}$ are known exactly.

\paragraph{SPlus.}
Suppose $\vecop(\bm M)$ is zero-mean Gaussian and the fixed rotation above
gives uncorrelated entries. These entries are then independent, and their
signs are independent, equally likely $+1$ and $-1$ values. Therefore
\[
 \bm U=\bm Q_L\operatorname{sign}(\widetilde{\bm M}_t)\bm Q_R^\top,
 \qquad \bm C=\bm I.
\]
Unlike SOAP, this transformation does not require estimating the coordinate standard deviations.

\paragraph{SPlus vs.\ SOAP Under Imperfect Decorrelation.}
The preceding analyses of SOAP and SPlus assume that the rotation decorrelates the entries of $\bm M_t$, with SOAP additionally having access to the true coordinate standard deviations. We can get further insight
by considering an imperfect rotation: fix the estimated $\bm Q_L,\bm Q_R$
before sampling $\vecop(\bm M_t)\sim\mathcal N(\bm 0,\bm\Sigma)$, and let
$\bm{\mathcal C}$ be the correlation matrix of the entries of
$\widetilde{\bm M}_t=\bm Q_L^\top\bm M_t\bm Q_R$.
To characterize how SPlus transforms these correlations, we apply the Gaussian arcsine law \citep{dirksen2022onebit}, which gives the covariance after taking signs. Rotating back yields
\begin{equation}
 \bm{C}_{\rm SPlus}=\bm{Q}\left[\frac2\pi\arcsin(\bm{\mathcal C})\right]\bm{Q}^\top,
 \qquad
 \left|\frac2\pi\arcsin(\rho)\right|\leq|\rho|.
 \label{rw:eq:arcsine}
\end{equation}
Here $\arcsin$ acts entrywise: the signs have unit variance in the rotated basis, and residual
correlations contract, but need not vanish.

With the same rotation and exact coordinate standard deviations, SOAP instead gives
$\bm C_{\rm SOAP}=\bm Q\bm{\mathcal C}\bm Q^\top$:
coordinate normalization leaves these correlations unchanged.
For both methods, the covariance in the rotated basis has unit diagonal.
Its squared Frobenius distance to identity is therefore just the sum of the
squared off-diagonal entries. Rotating back preserves this distance, yielding
\[
 \begin{aligned}
 \|\bm C_{\rm SPlus}-\bm I\|_F^2
 &=\sum_{i\ne j}\left[\frac2\pi\arcsin(\mathcal C_{ij})\right]^2\\
 &\leq\sum_{i\ne j}\mathcal C_{ij}^2
 =\|\bm C_{\rm SOAP}-\bm I\|_F^2.
 \end{aligned}
\]
Under these assumptions, SPlus produces an update covariance \emph{at least as close to the identity} in Frobenius norm as SOAP, even when the rotation does not perfectly decorrelate the proposal: its sign transformation contracts residual correlations without requiring coordinatewise variance estimates.

\paragraph{Muon.}
Suppose $\bm M$ is full rank almost surely, and its distribution is unchanged
by any orthogonal transformation $\bm M\mapsto\bm O\bm M$ when $m\geq n$
(or $\bm M\mapsto\bm M\bm O$ when $m<n$).
For the exact polar transformation with a fixed scale $\gamma>0$,
\[
 \bm U=\gamma\operatorname{polar}(\bm M),\qquad
 \bm C=\frac{\gamma^2}{\max(m,n)}\bm I.
\]
Thus, Muon produces an isotropic update covariance under strong distributional symmetry assumptions, but its polar transformation does not generally guarantee isotropy.

\paragraph{Relation to MuLoCo}
MuLoCo studies how Muon's normalized matrix updates affect pseudo-gradient aggregation \citep{therien2025muloco}. It argues that large singular values in individual update matrices contribute to misalignment during aggregation, an issue mitigated by using Muon as the inner optimizer. In contrast, our work examines the covariance structure of Muon's updates, a distribution-level property that provides a complementary perspective and highlights potential limitations of Muon's normalization. Furthermore, our error decomposition directly connects this covariance structure to the stochastic component of the error relative to the expected DP update.

\paragraph{Beyond one local step.}

Equations~\eqref{rw:eq:main-decomposition} and~\eqref{rw:eq:worst-risk}
apply to the complete displacement after any number of inner optimization
steps $H$. However, if $\bm{\Delta}_{r,h}$ denotes the displacement
at the $h$-th local step, the covariance of replica $r$'s complete
displacement is given by
\begin{equation*}
 \bm{C}_r=\sum_{h=1}^H\Cov(\bm{\Delta}_{r,h})
       +\sum_{h\ne k}\Cov(\bm{\Delta}_{r,h},\bm{\Delta}_{r,k}).
\end{equation*}
Consequently, whitening individual (one-step) local updates does not
by itself whiten the complete pseudo-gradient, as temporal covariances
between successive updates can introduce anisotropy.
Approximate isotropy of the complete pseudo-gradient can nevertheless
be maintained when the aggregate contribution of these temporal
covariances is sufficiently small.

\subsection{Single-Step Mean Mismatch and Second-Moment Estimation}
\label{rw:sec:h1}
We now consider the single-step setting ($H=1$), where all replicas
share the same initial parameters and optimizer states, and analyze
the mean mismatch term of the decomposition introduced in
Section~\ref{rw:sec:decomposition}.
This eliminates multi-step local drift, allowing us to study
the remaining mean mismatch $\bm{b}_R$ more easily.
At fixed global batch size $B$, let
$\widehat{\bm{g}}_r\coloneqq\vecop(\bm{G}_r)$
denote the independent and identically distributed replica gradients,
and define
\begin{equation}
 \bm{g}\coloneqq\E[\widehat{\bm{g}}_r],
 \qquad
 \bar{\bm{g}}\coloneqq\frac1R\sum_r\widehat{\bm{g}}_r,
 \qquad
 \bm{\Sigma}_g\coloneqq\Cov(\bar{\bm{g}}).
 \label{rw:eq:batch-model}
\end{equation}
Then $\Cov(\widehat{\bm{g}}_r)=R\bm{\Sigma}_g$.
For IID examples, $\bm{\Sigma}_g=\bm{\Omega}/B$, where $\bm{\Omega}$ is the per-example
covariance. Let $T(\widehat{\bm{g}};s)$ denote the parameter-update map that takes a gradient $\hat{\bm{g}}$ and incoming optimizer state $s$ and returns the corresponding parameter displacement. Then
\begin{equation}
 \bar{\bm{\Delta}}=\frac1R\sum_rT(\widehat{\bm{g}}_r;s),\qquad
 \bm{d}_1^{\rm DP}=\E[T(\bar{\bm{g}};s)],\qquad
 \bm{b}_R=\alpha_R\E[\bar{\bm{\Delta}}]-\bm{d}_1^{\rm DP}.
 \label{rw:eq:h1-reference}
\end{equation}
In this single-step setting, DP applies the optimizer transformation
$T$ to the averaged gradient $\bar{\bm{g}}$, whereas DiLoCo applies
$T$ independently to each replica's gradient before averaging the
resulting updates.
While an affine $T$ commutes with averaging sample-by-sample, sharing the incoming optimizer states does not generally make an adaptive optimizer's update transformation affine.

We first analyze SPlus, whose sign-based transformation permits explicit expressions for the mean and covariance of its updates.

\paragraph{SPlus: a direct bound on the mean term.}
Fix the shared SOAP basis before the step; this covers a step between refreshes
or a lagged-basis implementation. Momentum is still updated by
\eqref{rw:eq:momentum}. The average of the replica momenta is exactly the DP
momentum from the same incoming state. Let $\widetilde{m}_{r,t,j}$ denote entry
$j$ of $\vecop(\widetilde{\bm{M}}_{r,t})$, for $j=1,\ldots,d=mn$, and assume
\begin{equation}
 \widetilde{m}_{r,t,j}=\sigma_j(a_j+\sqrt R\,Z_{r,j}),
 \qquad Z_{r,j}\stackrel{\rm iid}{\sim}\mathcal N(0,1),\quad \sigma_j>0.
 \label{rw:eq:sign-model}
\end{equation}
Thus $a_j$ is the mean divided by the standard deviation of the \emph{DP}
momentum coordinate, not its replica-local SNR. We use unit step scale in the
following proposition and restore other deterministic scales multiplicatively.

\begin{proposition}[SPlus mean and covariance relative to DP]
\label{rw:prop:suds-dp}
Under \eqref{rw:eq:sign-model}, set $\bm{a}=(a_j)_{j=1}^d$,
$\psi(t)=2\Phi(t)-1$, and $c=\sqrt{2/\pi}$. With $\alpha_R=\sqrt R$,
\begin{equation}
 \begin{split}
 \bm{b}_R&=\bm{Q}\big[\sqrt R\,\psi(\bm{a}/\sqrt R)-\psi(\bm{a})\big],\\
 \|\bm{b}_R\|_2^2&\leq\frac{c^2}{36}\left(1-\frac1R\right)^2\sum_j a_j^6,\\
 \Cov(\sqrt R\,\bar{\bm{\Delta}})&=
 \bm{Q}\operatorname{Diag}\big(1-\psi(a_j/\sqrt R)^2\big)\bm{Q}^\top\preceq \bm{I}.
 \end{split}
 \label{rw:eq:suds-dp-bound}
\end{equation}
Here $\psi$ acts entrywise. In particular, every directional error in
\eqref{rw:eq:main-decomposition} is at most
$1+\frac{c^2}{36}(1-R^{-1})^2\sum_j a_j^6$.
\end{proposition}
The proof is given in Appendix~\ref{rw:app:suds-dp}. 
For a small vector $\bm{a}$, the mean mismatch is third order and its squared
contribution is sixth order. At $\bm{a}=0$, the calibrated local aggregate and DP
have the same mean and covariance, although their distributions differ. The
$\sqrt R$ calibration makes the stochastic term order one; it does not retain
an additional $1/R$ variance gain at fixed global batch.

\paragraph{Mean mismatch for Adam} We now bound the mean mismatch $\bm{b}_R$ for Adam.
At fixed global batch size, the local gradient variance is $R$ times that of
the DP average. Fresh-gradient second moments therefore change the
normalization even from a common incoming state. Proposition~\ref{rw:prop:beta-mean}
bounds the resulting \textbf{mean-update error}; Appendix~\ref{rw:app:state-and-curvature}
gives the second-moment identities and the Gaussian variance-differentiation
argument \citep{voigtlaender2021price}.

\begin{proposition}[Second-moment memory controls a one-step mean bound]
\label{rw:prop:beta-mean}
In a fixed parameter coordinate $j$, let the incoming Adam state be $(m,v)$, with
$v>0$ and $0<\beta_2<1$. Omit $\epsilon$, bias correction, and the learning rate,
and define the complete same-step scalar map
\begin{equation}
 T(x;m,v)=\frac{\beta_1m+(1-\beta_1)x}
                  {\sqrt{\beta_2v+(1-\beta_2)x^2}}.
 \label{rw:eq:adam-scalar}
\end{equation}
Suppose $\widehat{g}_r\sim\mathcal N(g,R\sigma^2)$ independently, so
$\bar{g}\sim\mathcal N(g,\sigma^2)$. With $\alpha_R=1$, write $b_{R,j}$ for this coordinate of the
bias $\bm{b}_R$ from \eqref{rw:eq:h1-reference}. It obeys
\begin{equation}
 |b_{R,j}|\leq
 \frac{(1-\beta_2)(R-1)\sigma^2}
      {2(\beta_2v)^{3/2}}
 \left(|\beta_1m|+3(1-\beta_1)|g|\right).
 \label{rw:eq:beta-mean-bound}
\end{equation}
For SOAP with a basis fixed before the fresh gradients, the same bound applies
to the corresponding coordinate of $\bm{Q}^\top\bm{b}_R$.
\end{proposition}
This is a finite-noise bound, proved in Appendix~\ref{rw:app:beta-mean}. For fixed $\beta_2$, the upper bound grows linearly with $R-1$, motivating an increase in $\beta_2$ as the number of replicas grows. Adapting the scaling rule of \citet{marek2025small} to our setting, we can preserve the second-moment EMA's half-life in terms of local training examples by setting:
\begin{equation}
 \beta_2(R)=\beta_2(R_0)^{R_0/R},\qquad 1-\beta_2(R)=O(R^{-1}).
 \label{rw:eq:beta-rule}
\end{equation}
If $m,v,g,\sigma$ are held fixed and
$\beta_2(R)$ follows \eqref{rw:eq:beta-rule}, with $\beta_2(R)$ bounded away from
zero, the right-hand side remains bounded rather than growing linearly in $R$.
For each $R$, both routes in this comparison use the \emph{same} $\beta_2(R)$.
\textbf{Scaling second-moment memory therefore controls a bound on the Adam
mean mismatch as $R$ grows.} Summing the squared coordinate bounds controls
$\|\bm{b}_R\|_2^2$, and hence the mean term
$(\bm{q}^\top\bm{b}_R)^2$ for every unit $\bm{q}$ in
\eqref{rw:eq:main-decomposition}. The same conclusion holds after SOAP's
orthogonal rotation back to parameter coordinates.
This conditional one-step result motivates the
matched-$\beta_2$ experiments in Section~\ref{sec:beta2-results}.

\paragraph{Scale-estimation error at smaller local batches.}
Beyond the mean mismatch analyzed above, smaller local batches
can also increase the variability of second-moment estimates.
Indeed, at fixed $\beta_2$, smaller local batches produce noisier
gradients, which can increase fluctuations in the estimated
second moment $\widehat v$.
For SOAP, these fluctuations in $\widehat v$ change the coordinatewise normalization and hence the update covariance.
SPlus avoids this source of variability by taking signs in its estimated basis rather than dividing by estimated coordinate scales.

\input{sections/experiments}

\bibliographystyle{preprint}
\bibliography{main}

\numberwithin{equation}{section}
\appendix
\clearpage
\input{appendices/essential_results}

\end{document}

%% file: sections/experiments.tex
\section{Empirical Evaluation}
\label{sec:empirical-evaluation}

We first compare replica scaling across inner optimizers (Sec.~\ref{sec:structured-scaling}), then isolate the effect of second-moment memory on performance (Sec.~\ref{sec:beta2-results}), and finally examine the update distributions behind these results (Sec.~\ref{sec:fixed-covariance-probe}). The analysis connects the observed scaling to the two terms in Equation~\eqref{rw:eq:main-decomposition}: mean mismatch and stochastic aggregation error.

\paragraph{Experimental setup.}
\label{sec:setup}
Our main experiments train 178M- and 512M-parameter LLaMA models on DCLM~\citep{li2024dclm} with sequence length 2048. We evaluate performance on a held-out 100M-token validation split. The 178M runs use a fixed global batch of 524K tokens and approximately 3.55B training tokens; the 512M runs use 2M tokens per global batch and 10.25B-token budget. Unless stated otherwise, replicas take $H=30$ local steps between aggregations and all local optimization methods use a Nesterov SGD outer optimizer. Appendix~\ref{sec:hyperparameters} records the architecture, training budgets, and result-selection protocol.

\subsection{Comparing Inner Optimizers in Replica Scaling}
\label{sec:structured-scaling}

We compare SOAP, SPlus, Muon, AdamW, and Shampoo as inner optimizers for DiLoCo while scaling the number of replcas. Figures~\ref{fig:structured-scaling}
and~\ref{fig:structured-scaling-512m} show both absolute validation loss and
the percentage change from each optimizer's own DP reference \citep{therien2025muloco}.
The latter allows making a direct comparison between inner optimizers as it relates to their impact on worker scaling without being confounded by absolute performance differences of the optimizers (e.g., accounting for the fact that Moun has lower absolute loss than AdamW).

\paragraph{178M Experiment Scale} At 178M scale (Fig~\ref{fig:structured-scaling}), we observe, as in \cite{therien2025muloco}, that Muon has better replica scaling than AdamW, not just DP performance. SOAP has both stronger DP performance and significantly improved replica scaling compared to Muon and Adam. SPlus has the strongest worker scaling, leading to the smallest increase over its own DP reference and the lowest observed loss at large $R$.
At $R=64$, the loss change is \textbf{$+3.1\%$ for SPlus and $+4.8\%$ for
SOAP}, compared with $+8.5\%$ for Muon and $+11\%$ for AdamW.
Shampoo improves on AdamW but remains behind
SOAP and SPlus. As discussed in our work, it does not assure an approximately whitened update unlike SOAP and SPlus. Thus, both the normalization and the coordinates in which it is applied matter empirically. Section~\ref{sec:fixed-covariance-probe} examines how this ordering relates to the covariance of the returned updates. Interestingly, at 178M, $R=8$ SPlus actually improves performance over its DP baseline. This is likely because at $R>1$ it can mimic a majority vote-based technique~\citep{bernstein2018signsgd}.

\begin{figure}[H]
\centering
\includegraphics[width=\linewidth]{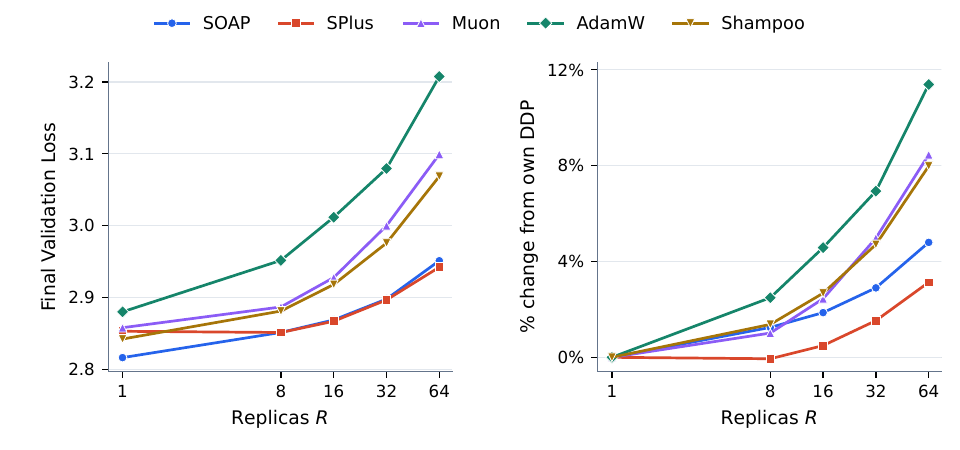}
\caption{\textbf{178M replica scaling.} $R=1$ reports DP performance, while $R>1$ reports the performance of local optimization with synchronization interval $H=30$. Left: final validation loss. Right: change in loss from the same optimizer's own DP (R=1) baseline. We observe that matrix-based optimizers have stronger worker scaling than AdamW and that SPlus achieves the strongest scaling of all relative to its DP baseline. }
\label{fig:structured-scaling}
\end{figure}

\Needspace{9\baselineskip}
\paragraph{512M Experiment Scale}
\label{sec:structured-scaling-512m}
In Figures~\ref{fig:structured-scaling-512m}, we observe a similar pattern. SOAP is, once again, the strongest DP reference and has the lower loss through $R=32$. SPlus overtakes it at $R=64$ and $R=96$.
At $R=96$, SPlus' los change is \textbf{$+2.1\%$ above its own DP loss}, versus $+4.4\%$
for SOAP, $+7.7\%$ for Muon, and $+12\%$ for AdamW.
The crossover is consistent across model scales: SOAP's DP performance is stronger, whereas SPlus degrades less as replicas are scaled and the local batch size shrinks. Both SPlus and SOAP show drastically better scaling than Adam and Muon, making them practical alternatives when training with high worker counts. 

\begin{figure}[H]
\centering
\includegraphics[width=\linewidth]{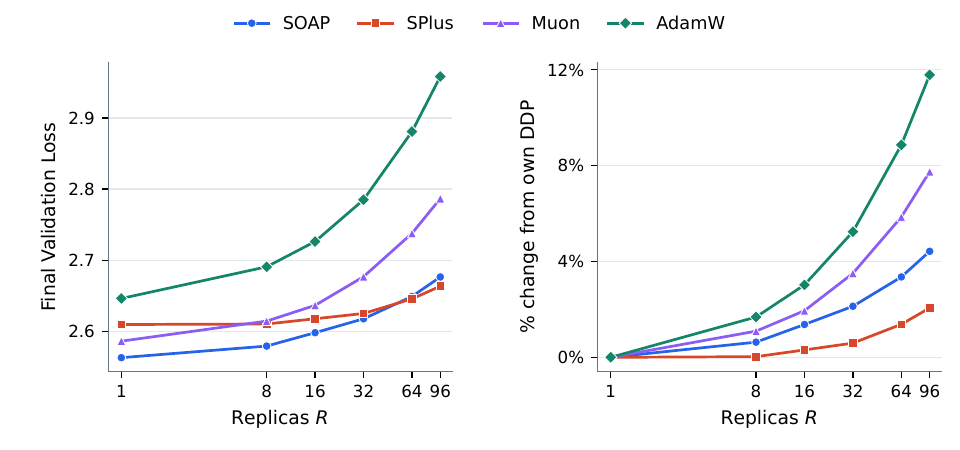}
\caption{\textbf{512M replica scaling.} $R=1$ reports DP performance, while $R>1$ reports the performance of local optimization with synchronization interval $H=30$.
Left: final validation loss. Right: change in loss from the same optimizer's own DP (R=1) baseline. As in Fig.~\ref{fig:structured-scaling}, we observe that matrix-based optimizers have stronger worker scaling than AdamW and that SPlus achieves the strongest scaling of all relative to its DP baseline.}
\label{fig:structured-scaling-512m}
\end{figure}

\begin{table}[H]
\centering
\begin{minipage}[t]{.46\linewidth}
\vspace{0pt}
\caption{\textbf{178M final validation loss.} DP uses $R=1$; other runs use $H=30$. Best in each column is bold.}
\label{tab:structured-scaling}
\centering\small
\setlength{\tabcolsep}{3pt}
\renewcommand{\arraystretch}{1.15}
\begin{tabular*}{\linewidth}{@{\extracolsep{\fill}}lrrrrr@{}}
\toprule
& & \multicolumn{4}{c}{Replicas $R$}\\
\cmidrule(lr){3-6}
Optimizer & DP & 8 & 16 & 32 & 64\\
\midrule
SOAP & \textbf{2.816} & 2.851 & 2.869 & 2.898 & 2.951\\
SPlus & 2.853 & \textbf{2.851} & \textbf{2.867} & \textbf{2.897} & \textbf{2.942}\\
Muon & 2.858 & 2.887 & 2.928 & 3.000 & 3.099\\
AdamW & 2.880 & 2.952 & 3.012 & 3.080 & 3.208\\
Shampoo & 2.842 & 2.881 & 2.918 & 2.976 & 3.069\\
\bottomrule
\end{tabular*}
\end{minipage}\hfill
\begin{minipage}[t]{.51\linewidth}
\vspace{0pt}
\caption{\textbf{512M final validation loss.} DP uses $R=1$; other runs use $H=30$. Best in each column is bold.}
\label{tab:structured-scaling-512m}
\centering\small
\setlength{\tabcolsep}{3pt}
\renewcommand{\arraystretch}{1.15}
\begin{tabular*}{\linewidth}{@{\extracolsep{\fill}}lrrrrrr@{}}
\toprule
& & \multicolumn{5}{c}{Replicas $R$}\\
\cmidrule(lr){3-7}
Optimizer & DP & 8 & 16 & 32 & 64 & 96\\
\midrule
SOAP & \textbf{2.563} & \textbf{2.579} & \textbf{2.598} & \textbf{2.618} & 2.649 & 2.676\\
SPlus & 2.610 & 2.610 & 2.618 & 2.625 & \textbf{2.646} & \textbf{2.663}\\
Muon & 2.586 & 2.615 & 2.637 & 2.677 & 2.738 & 2.787\\
AdamW & 2.646 & 2.691 & 2.726 & 2.785 & 2.881 & 2.958\\
\bottomrule
\end{tabular*}
\end{minipage}
\end{table}

\paragraph{Comparing results across 178M and 512M} Prior work has demonstrated that the worker scaling performance of AdamW and Muon improves as data and model scale are increased~\citep{therien2025muloco,charles2025scaling}. In Figure~\ref{fig:scaling-178m-vs-512m}, we compare the performance of SPlus, SOAP, AdamW, and Muon when scaling replicas. We also observe that scale improves replica scaling across all optimizers can significantly benefit from this including SOAP, SPlus, and Muon.

\paragraph{Isolating the Effect of Local Steps: \textit{Replica-scaling effects are visible at $H=1$, before multi-step
parameter drift can occur!}} Unlike SGD, in adaptive optimizers, the divergence to standard DP training occurs even with a single step of local optimization (due to the non-linear map of gradients to update). Indeed, the arguments regarding covariance structure in Sec.~\ref{rw:sec:theory} apply directly to this case. To isolate the effects of local drift when taking multiple steps from the impact of the adaptive optimizer update, we perform scaling experiments using a single local update. We also remove the impacts of the outer optimizer by using a standard SGD outer optimizer with a learning rate of 1 (equivalent to parameter averaging). Table~\ref{tab:178m-h1-current} reports the results of this experiment. We observe similar results to the multi-step setting: Adam, Muon, and Shampoo have poorer replica scaling, while SOAP and SPlus remain substantially more robust. This is the setting in which Section~\ref{rw:sec:h1} isolates the effect of applying the optimizer before gradient averaging. We also observe that performance (even in terms of scaling with respect to DP) is far more impacted by the choice of optimizer and the number of replicas than the number of local steps. This suggests the impact of the local optimizer on the update and subsequent aggregation can be more significant than the parameter drift often studied in local optimization \citep{stich2019local,karimireddy2020scaffold}.

\begin{table}[H]
\centering\small
\caption{\textbf{178M performance using H=1 (single-step) local updates.} We report the final validation loss for each method. To isolate the effect of the adaptive optimizer's update, we use pure parameter averaging at each step and persistent local optimizer states. All methods train for $6,775$ steps. The DP, $R=8$, and $R=64$ methods are directly tuned, while the $R=16,32$ configurations use interpolated hyperparameters rather than separate full sweeps. The best-performing method in each column is bolded.}
\label{tab:178m-h1-current}
\setlength{\tabcolsep}{5pt}
\begin{tabular}{lrrrrr}
\toprule
Inner optimizer & DP & $R=8$ & $R=16$ & $R=32$ & $R=64$\\
\midrule
SOAP & \textbf{2.816} & 2.842 & 2.864 & 2.901 & 2.958\\
SPlus & 2.853 & \textbf{2.837} & \textbf{2.853} & \textbf{2.873} & \textbf{2.912}\\
Muon & 2.858 & 2.979 & 3.073 & 3.230 & 3.417\\
AdamW & 2.880 & 2.932 & 2.969 & 3.050 & 3.157\\
Shampoo & 2.842 & 2.899 & 2.947 & 3.002 & 3.097\\
\bottomrule
\end{tabular}
\end{table}

Interestingly, Muon is less robust here than at $H=30$ and is worse than AdamW at the larger replica counts. Its strong $H=30$ results therefore do not transfer unchanged to pure one-step averaging. We hypothesise that this may be due to differences in their outer optimizer, which could be more favourable for some methods and less for others. 

\subsection{Second-moment memory and replica scaling}
\label{sec:beta2-results}

Following Sec.~\ref{rw:sec:h1}, we next study a practical consequence of smaller local batches for Adam-based methods. Equation~\eqref{rw:eq:expected-state-gap} shows how local gradient variance enters the second-moment discrepancy.
The token-half-life rule in Equation~\eqref{rw:eq:beta-rule} increases
$\beta_2$ as $R$ grows, preserving the number of tokens represented by
the running average. In this section, we study the impact of applying this and, in particular, evaluate how it impacts conclusions in prior work~\citep{therien2025muloco,horoi2026isoloco}. This rule is already incorporated in the main results in Sec.~\ref{sec:structured-scaling}. Here we evaluate both the Adam-based DiLoCo optimizer and IsoLoCo which is a modification of the outer optimizer. We incorporate it into our analysis as it is based on the Adam inner optimizer, which can benefit from appropriate $\beta_2$ tuning. We note that in the prior section's experiments, to ensure the strongest baseline, this rule is already used by construction in all optimizers with a second moment. This section, on the other hand, emphasizes the improvement compared to existing literature, which does not scale $\beta_2$ when scaling R \cite{charles2025scaling,horoi2026isoloco,sarfi2025sparseloco}

To empirically measure the importance of second moment estimation highlighted in Sec.~\ref{rw:sec:h1}, we compared the performance of optimizers using a fixed $\beta_2$ and tuning it at every replica scale. Figure~\ref{fig:g-beta-178m} reports the final validation loss of DiLoCo, IsoLoCo, and MuLoCo under these strategies. We note that while IsoLoCo and DiLoCo both apply AdamW to each tensor, MuLoCo only applies it to non-hidden tensors. We observe all optimizers improve from tuning $\beta_2$ as replicas increase, but IsoLoCo and DiLoCo benefit more. Under $\beta_2$ tuning, we find that IsoLoCo, despite using an AdamW inner optimizer, can even outperform MuLoCo at large worker counts.

\begin{figure}[H]
\centering
\begin{minipage}[t]{.56\linewidth}
\vspace{0pt}
\centering
\includegraphics[width=\linewidth]{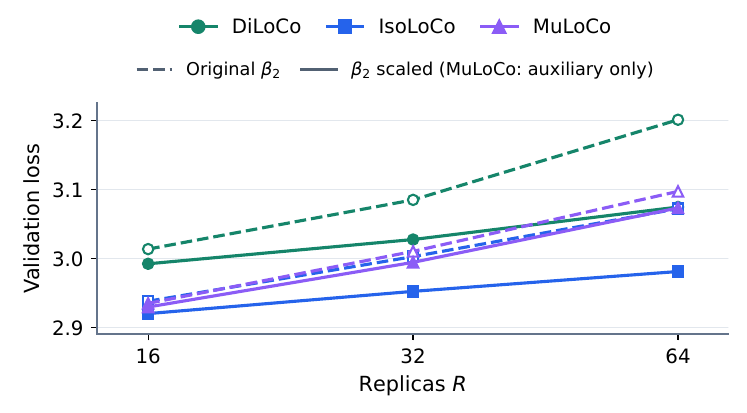}
\caption{\textbf{The effect of tuning AdamW's $\beta_2$ in DiLoCo, IsoLoCo, and MuLoCo. } We report the performance of each method at 178M scale when fixing (dashed lines) and tuning (solid lines) $\beta_2$ at each worker count. While IsoLoCo and DiLoCo both apply AdamW to each tensor, MuLoCo only applies it to non-hidden layers. We observe all optimizers improve from this change, but IsoLoCo and DiLoCo benefit significantly more, revealing that IsoLoCo can outperform MuLoCo at larger worker counts when $\beta_2$ is well tuned.   }
\label{fig:g-beta-178m}
\end{minipage}\hfill
\begin{minipage}[t]{.41\linewidth}
\vspace{0pt}
\centering
\includegraphics[width=\linewidth]{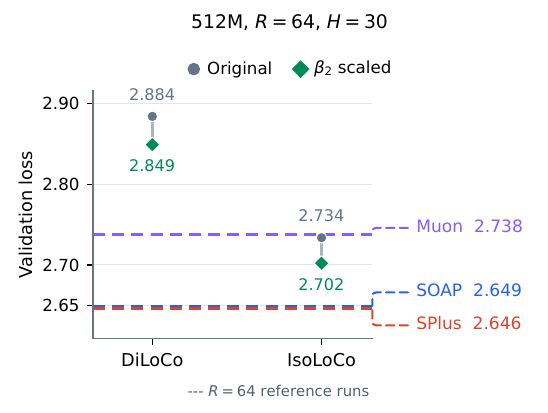}
\caption{\textbf{Tuning $\beta_2$ for larger model and replica.} We ablate using fixed $\beta_2$ vs scaling at high replica count (R=64) and 512M model scale. We observe that performance consistently improves for DiLoCo and IsoLoco. Notably IsoLoco outerperforms Muon inner optimizer (MuLoco). }
\label{fig:g-beta-512m}
\label{tab:512m}
\end{minipage}
\end{figure}

We also confirm that this scaling rule works well at 512M in Figure \ref{fig:g-beta-512m}. Both methods use the published $R=8$ hyperparameters without retuning.
At 512M and $R=64$, changing only $\beta_2$ reduces DiLoCo loss from
$2.884$ to $2.849$ and IsoLoCo loss from $2.734$ to $2.702$,
approximately $1.2\%$ in each case.
\emph{}{Second-moment memory is thus an effective control on replica scaling, including for IsoLoCo, whose outer optimizer does not remove this inner AdamW sensitivity.} Notably, after the relevant control, IsoLoCo, which uses an AdamW inner optimizer, can outperform MuLoCo, while only doing the spectral decomposition during the outer step.

\paragraph{Discussion of Practical Implications} 
Our results suggest that local optimization methods can be deployed across a large number of replicas (50+) while maintaining minimal performance degradation of $<3\%$ when utilizing SOAP or SPlus. Notably, these methods consume more memory than Muon due to the need to maintain $L$ and $R$, however, scalable recipes for SOAP have been studied~\citep{khona2026soapmuon}, demonstrating that it can perform well at very large scale. A particularly relevant setting for our work is cross-data center training where each replica has a considerable number of chips that can be used to shard the additional optimizer states and perform the optimizer step in parallel. At 178M and 512M scales, our results suggest that at a smaller number of replicas SOAP is the preferred inner optimizer, while at a large number of replicas, SPlus performs best. Another observation of our work is that the recently proposed IsoLoCo outer optimizer, which has nearly the same overhead as AdamW inner optimizer-based DiLoCo, can perform very effectively at a large number of replicas simply by adjusting its $\beta_2$ according to the local batch size. When doing so, we show that it outperforms MuLoCo. This can make it an excellent candidate for settings where each replica has limited memory or existing infrastructure and training recipe are already well designed for Adam.

\subsection{Analysis: covariance structure and mean mismatch}
\label{sec:fixed-covariance-probe}

We empirically analyze the covariance structure and mean mismatch discussed in Section 4 for the different optimizer classes. 

\paragraph{Empirically Measuring Whitening}
Using the 178M training setting and beginning with checkpoints at outer round 75, we empirically estimate how close to whitened the covariance of the pseudo-gradient is for different inner optimizers. We perform 64 instantiations of inner training loop starting with the same checkpoint and inner optimizer state, thus only varying the data used. Subsequently we use pairs to construct an estimator for the following quantities:
\begin{equation}
 \widetilde{\bm C}_h:=\frac1{R^2}\sum_{r=1}^R\bm C_{r,h},
 \qquad
 a_h=\left(\frac{d\,\tr(\widetilde{\bm C}_h^2)}
                     {\tr(\widetilde{\bm C}_h)^2}-1\right)^{1/2}.
\label{eq:probe-metrics}
\end{equation}
Here $d$ is the probed dimension; $a_h=0$ means isotropy, and larger
values indicate greater anisotropy, independently of overall noise
magnitude. The constructed $\widetilde{\bm C}_h$ equals the covariance
of the averaged update when cross-worker noise covariances vanish,
conditional on the starting parameters and optimizer states.
Appendix~\ref{app:fixed-covariance-estimator} gives the estimators and exact replay protocol. We chose this particular quantity as a proxy for isotropy as the trace and square trace have efficient estimators available in the statistical literature. The results are shown in Table~\ref{tab:fixed-covariance} below.

\begin{table}[H]
\centering
\begin{minipage}[t]{.35\linewidth}
\vspace{0pt}
\centering\small
\caption{Update-noise anisotropy at $R=64$, round 75, using all 64 replicas and local batch four. Lower is more isotropic.}
\label{tab:fixed-covariance}
\renewcommand{\arraystretch}{1.15}
\begin{tabular*}{\linewidth}{@{\extracolsep{\fill}}lrr@{}}
\toprule
Optimizer & $a_1\;\downarrow$ & $a_{30}\;\downarrow$\\
\midrule
AdamW & 35 & 11\\
Muon & 3.8 & 3.3\\
SOAP & 2.3 & 2.5\\
SPlus & \textbf{0.43} & \textbf{2.4}\\
\bottomrule
\end{tabular*}
\end{minipage}\hfill
\begin{minipage}[t]{.62\linewidth}
\vspace{0pt}
\centering\small
\caption{Mean mismatch is smaller than stochastic error in the detected high-noise directions, and particularly small for SPlus and SOAP compared with AdamW.}
\label{tab:g-mean-noise}
\setlength{\tabcolsep}{3pt}
\renewcommand{\arraystretch}{1.15}
\begin{tabular*}{\linewidth}{@{\extracolsep{\fill}}lrrr@{}}
\toprule
Optimizer & \shortstack{Mean mismatch\\($\times10^{-6}$)}
& \shortstack{Stochastic error\\($\times10^{-6}$)}
& \shortstack{Mean / noise\\($\%$)}\\
\midrule
SPlus & \textbf{.044} & \textbf{.91} & 7.9\\
SOAP & .12 & 1.4 & 7.5\\
Muon & .11 & 1.1 & 14\\
AdamW & 5.4 & 19 & 19\\
\bottomrule
\end{tabular*}
\end{minipage}
\end{table}

We observe that \textit{covariance structure broadly follows the scaling results.} with optimizers that have better replica scaling having better covariance structure. 
The anisotropy differences are already pronounced after one
local step, before parameter drift across replicas can affect subsequent
gradients.  $a_1$ ranges from $35$ for AdamW to $0.43$
for SPlus, with Muon and SOAP between them. After 30 steps, SOAP and
SPlus remain less anisotropic than Muon and AdamW, broadly following
their replica-scaling performance. Interestingly, Muon is significantly more isotropic than Adam, helping to explain its improved scaling \cite{therien2025muloco}.
The next measurements examine the curvature of high-variance noise
directions, connecting this shape comparison to optimization.

\paragraph{Noise dominates mean mismatch along high-noise directions.}
Finally, we compare the two terms of
Equation~\eqref{rw:eq:main-decomposition} along directions selected for
high update variance. For each optimizer, all replicas and its DP
reference start from the same checkpoint weights and replica 0's saved
optimizer state. They run $H=30$ steps on paired data: DP uses the
union of the local batches at each step. We analyze the 63 hidden
attention and MLP weight matrices separately, excluding embeddings,
the output head, and other auxiliary parameters.

For each matrix, 16 repetitions fit a scalar $\alpha_R$ and a leading
noise direction $\widehat{\bm q}$; 24 additional repetitions estimate
the directional mean mismatch $(\widehat{\bm q}^{\top}\bm b_R)^2$ and
stochastic error $\alpha_R^2\widehat{\bm q}^{\top}\bm C_{\rm agg}\widehat{\bm q}$.
Both terms are divided by the squared norm of that matrix's DP mean
update. This per-matrix calibration applies the decomposition to each
matrix separately; it is not a change to the training learning rates.
Appendix~\ref{app:g-directional-estimator} defines the estimators,
including the paired finite-sample correction.

We observe \emph{the mean-mismatch term is only $7.5$--$19\%$ of the stochastic
term on average along these directions.} This supports controlling
the stochastic term in Equation~\eqref{rw:eq:main-decomposition} which is the focus of our analysis in Section 4, and allow to better confirm that the covariance structure of SPlus and SOAP lead to their improved replica scaling.

\paragraph{Extreme noise also lies in high-curvature directions.}
We have seen bounding the maximum noise direction of the covariance as in Equation~\ref{rw:eq:worst-risk} can reduce the maximum discrepancy to the DDP direction, in particular since the above empirical experiments suggest it dominates the bound on the error with respect to the DDP target.   
We next ask whether directions carrying extreme noise variance are directly responsible for poor performance of methods such as AdamW which do not bound them well. We do this by estimating these direction and then the underlying curvature. At each method's own checkpoint for outer step 75, we fit a unit
direction $\widehat{\bm q}$ that maximizes sample variance over 64
paired one-step update differences.  With $\bm{\mathcal H}$ the loss Hessian on 64 held-out
sequences, Table~\ref{tab:g-curvature} reports
\[
 \kappa(\widehat{\bm q})=
 \frac{\widehat{\bm q}^{\top}\bm{\mathcal H}\widehat{\bm q}}
 {\frac1{16}\sum_{j=1}^{16}\bm v_j^{\top}\bm{\mathcal H}\bm v_j},
\]
where the $\bm v_j$ are independent random unit vectors in the same
six-matrix subspace. A ratio of $100$ means $100$ times the average
random-direction \emph{curvature}, not noise variance.

\Needspace{20\baselineskip}
\begin{wraptable}{r}{.47\linewidth}
\vspace{-8pt}
\centering\small
\caption{Loss curvature along a fitted high-noise direction, relative
to random directions. Ranges span replicas 0 and 1, not confidence
intervals; local batch $b=256/R$.}
\label{tab:g-curvature}
\setlength{\tabcolsep}{4pt}
\begin{tabular}{@{}lrr@{}}
\toprule
&\multicolumn{2}{c}{Curvature ratio $\kappa(\widehat{\bm q})$}\\
\cmidrule(lr){2-3}
Optimizer & $R=8$ & $R=64$\\
\midrule
AdamW & $86$--$170$ & $460$--$730$\\
Muon  & $6.4$--$7.7$ & $47$--$100$\\
SOAP  & $3.1$--$3.2$ & $4.2$--$5.2$\\
SPlus  & $\bm{1.5}$--$\bm{1.6}$ & $\bm{1.5}$\\
\bottomrule
\end{tabular}
\vspace{-6pt}
\end{wraptable}

Appendix~\ref{app:g-curvature-estimator} details the estimator used. We observe that Adam and Muon have very high curvature in the directions with the most noise compared to average directions, in particular as we scale the replica size, while SOAP and SPlus are able to effectively redistribute noise in any particular direction due to their whitening. High noise along high curvature directions can be particularly problematic as small errors can be very damaging for the loss in high curvature regions. This gives an additional mechanism for understanding why whitening (which redistributes error more uniformly along all directions) can be particularly important in this setting.

\FloatBarrier
\section{Conclusion}
We have shown that differences to the DP pseudogradient can be seen to result from two primary sources of error: mean mismatch and stochastic error. Building on these insights, we showed that covariance structure in optimizers like AdamW is very important. In our empirical study, we establish the performance of different inner optimizers (SOAP, SPlus, Muon, AdamW), which have more and less appealing properties for reducing stochastic error according to our theory. We find that SPlus reaches the strongest performance at higher worker counts, and the best overall scaling with respect to DP, while SOAP reaches the best DP performance and seems on trend to become the best inner optimizer at scale. We then directly test our theoretical claims by calculating the mean mismatch and stochastic error for each optimizer, showing that these quantities are directly correlated with the observed performance trends.  

\section*{Acknowledgements}
We acknowledge support from NSERC Discovery RGPIN-2021-04104 and compute resources from Digital Research Alliance Canada.

%% file: appendices/essential_results.tex
Appendix~\ref{g:app:main} collects the supporting theoretical results and proofs.
Appendix~\ref{sec:hyperparameters} records the training setup, and
Appendix~\ref{app:g-mechanism-details} gives the analysis protocols and estimators.

\section{Supporting theoretical results}
\label{g:app:main}
\label{rw:sec:appendix}

\input{appendices/optimizer_covariance}

\input{appendices/mean_comparisons}

\input{appendices/basis_theory}

\input{appendices/training_details}

\input{appendices/analysis_protocols}

%% file: appendices/optimizer_covariance.tex
\subsection{Muon covariance under rotational proposals}
\label{app:muon-polar-covariance}

Let $p=\min(m,n)$ and $d_{\max}=\max(m,n)$. If a full-rank proposal has SVD
$\bm{M}=\widetilde{\bm{U}}\operatorname{diag}(\sigma_i)\widetilde{\bm{V}}^\top$, the
ideal Muon update with prescribed matrix-size scale $\gamma>0$ is
\begin{equation}
 \bm{U}_{\rm Muon}=\gamma\,\operatorname{polar}(\bm{M})
 =\gamma\widetilde{\bm{U}}\widetilde{\bm{V}}^\top.
 \label{eq:muon}
\end{equation}
The polar map discards the realized singular values and keeps the polar factor.
In particular,
$\operatorname{polar}(a\bm{M})=\operatorname{polar}(\bm{M})$ for every $a>0$,
so a positive multiplicative change to the entire proposal leaves the update
unchanged.

\begin{proposition}[Muon covariance under rotational proposals]
\label{prop:muon-rotational-covariance}
Suppose $\bm{M}_r$ is full rank almost surely. If $m\geq n$, suppose
$\bm{O}\bm{M}_r\overset{d}{=}\bm{M}_r$ for every $\bm{O}\in\mathrm{O}(m)$;
if $m<n$, suppose $\bm{M}_r\bm{O}\overset{d}{=}\bm{M}_r$ for every
$\bm{O}\in\mathrm{O}(n)$. Either one-sided condition suffices when $m=n$.
Then
\begin{equation}
 \E[\bm{U}_{{\rm Muon},r}]=\bm{0},
 \qquad
 \Cov(\vecop(\bm{U}_{{\rm Muon},r}))
 =\frac{\gamma^2}{d_{\max}}\bm{I}_{mn}.
 \label{eq:muon-rotational-covariance}
\end{equation}
Consequently, independent replicas have a uniformly averaged update covariance
$\gamma^2\bm{I}_{mn}/(d_{\max}R)$, even when their proposal singular-value
distributions differ.
\end{proposition}

\begin{proof}
Consider
$m\geq n$ and write the columns of
$\bm{Q}_r=\operatorname{polar}(\bm{M}_r)$ as
$\bm{q}_{r,1},\ldots,\bm{q}_{r,n}$. Equivariance of the polar map and left
invariance of the proposal law imply
$\bm{O}\bm{Q}_r\overset{d}{=}\bm{Q}_r$ for every
$\bm{O}\in\mathrm{O}(m)$. Taking $\bm{O}=-\bm{I}$ gives
$\E[\bm{Q}_r]=\bm{0}$. Moreover,
$\E[\bm{q}_{r,j}\bm{q}_{r,k}^\top]$ commutes with every orthogonal matrix and
must therefore equal $a_{jk}\bm{I}_m$. Since
$\bm{Q}_r^\top\bm{Q}_r=\bm{I}_n$,
\begin{equation}
 m a_{jk}
 =\E[\bm{q}_{r,k}^\top\bm{q}_{r,j}]
 =\mathbf 1_{\{j=k\}}.
\end{equation}
The column blocks of
$\E[\vecop(\bm{Q}_r)\vecop(\bm{Q}_r)^\top]$ are consequently
$\mathbf 1_{\{j=k\}}\bm{I}_m/m$, which proves
Equation~\eqref{eq:muon-rotational-covariance} because $d_{\max}=m$. The proof for
$m<n$ applies the same argument to the rows, using right invariance and
$\bm{Q}_r\bm{Q}_r^\top=\bm{I}_m$.
\end{proof}

%% file: appendices/mean_comparisons.tex
\subsection{Proof of Proposition~\ref{rw:prop:suds-dp} (SPlus to DDP comparison)}
\label{rw:app:suds-dp}

\begin{proof}[Proof of Proposition~\ref{rw:prop:suds-dp}]
Write $\psi(a)=2\Phi(a)-1$ and $c=\sqrt{2/\pi}$.
For one coordinate, set $X_r=\sigma(a+\sqrt R Z_r)$ and
$\bar{X}=R^{-1}\sum_rX_r$. Then
\[
 \Pr(X_r>0)=\Phi(a/\sqrt R),\quad
 \E[\sign(X_r)]=\psi(a/\sqrt R),\quad
 \E[\sign(\bar{X})]=\psi(a).
\]
In the shared-state model the last expression is exactly the mean of the signed
DDP momentum coordinate, because momentum is affine before the sign.
The expectation of the local average equals the single-replica expectation.
For $a\geq0$,
\begin{equation}
 \sqrt R\psi(a/\sqrt R)-\psi(a)
 =c\int_0^a\left(e^{-u^2/(2R)}-e^{-u^2/2}\right)du.
 \label{rw:eq:sign-integral}
\end{equation}
The integrand is nonnegative and at most
$u^2(1-R^{-1})/2$, by the unit Lipschitz bound for $e^{-x}$ on $x\geq0$.
Integration and odd symmetry give
\[
 \left|\sqrt R\psi(a/\sqrt R)-\psi(a)\right|
 \leq\frac c6(1-R^{-1})|a|^3.
\]
Squaring and summing over coordinates is valid because $\bm{Q}$ is orthogonal.
Each sign has variance $1-\psi(a_j/\sqrt R)^2$. Independence of projected
coordinates and of replicas gives the covariance in
\eqref{rw:eq:suds-dp-bound}. Finally use Proposition~\ref{rw:prop:decomposition}.
\end{proof}

The DDP covariance is
$\bm{Q}\operatorname{Diag}(1-\psi(a_j)^2)\bm{Q}^\top$. For $R\geq1$ the calibrated local
covariance minus this matrix is positive semidefinite, with operator norm at
most $c^2\max_j a_j^2$. Thus mean agreement to first order should not be read
as a variance improvement over signed DDP. At zero signal the covariances agree
exactly. 
\subsection{Adam/SOAP second moments and the mean comparison}
\label{rw:app:state-and-curvature}

Condition on common incoming optimizer states and parameters, as in
Section~\ref{rw:sec:h1}. Let $t$ denote the fresh-gradient step.
Averaging squared gradients is not the same as squaring their average:
the difference is the scatter of the replica gradients. Define
$\bm{S}_R=R^{-1}\sum_r(\widehat{\bm{g}}_r-\bar{\bm{g}})(\widehat{\bm{g}}_r-\bar{\bm{g}})^\top$.
For the second moments in \eqref{rw:eq:adam-map} and \eqref{rw:eq:soap-map},
\begin{equation}
 \begin{split}
 \frac1R\sum_r \bm{v}_{A,r,t}-\bm{v}_{A,\rm DDP,t}
   &=(1-\beta_2)\operatorname{diag}(\bm{S}_R),\\
 \frac1R\sum_r \bm{v}_{S,r,t}-\bm{v}_{S,\rm DDP,t}
   &=(1-\beta_2)\operatorname{diag}(\bm{Q}^\top \bm{S}_R\bm{Q}).
 \end{split}
 \label{rw:eq:state-gaps}
\end{equation}
The second identity assumes a common basis fixed before the fresh gradients;
$\bm v_A,\bm v_S$ are the vectorized buffers.
Thus locally updated second moments already differ from DDP before any
parameter drift.

Under the main-text batch model,
$\Cov(\widehat{\bm g}_r)=R\bm\Sigma_g$ and the replica gradients are
independent, so $\bm\Sigma_g$ is the covariance of the DDP-averaged
gradient. At fixed $B$, taking expectations in \eqref{rw:eq:state-gaps} gives
\begin{equation}
 \E\left[\frac1R\sum_r\bm{v}_{A,r,t}-\bm{v}_{A,\rm DDP,t}\right]
 =(1-\beta_2)(R-1)\operatorname{diag}(\bm{\Sigma}_g).
 \label{rw:eq:expected-state-gap}
\end{equation}
SOAP has the same expected identity in its fixed basis.

Expanding squares around the sample mean gives
\[
 \frac1R\sum_r\widehat{\bm{g}}_r^{\odot2}-\bar{\bm{g}}^{\odot2}
 =\frac1R\sum_r(\widehat{\bm{g}}_r-\bar{\bm{g}})^{\odot2}
 =\operatorname{diag}(\bm{S}_R).
\]
Multiplying by $1-\beta_2$ proves the Adam identity in
\eqref{rw:eq:state-gaps}; applying it to $\bm{Q}^\top\widehat{\bm{g}}_r$ proves the SOAP
identity when $\bm{Q}$ is fixed and common. Under \eqref{rw:eq:batch-model},
$\E[\bm{S}_R]=(R-1)\bm{\Sigma}_g$, which proves \eqref{rw:eq:expected-state-gap}.

For a scalar smooth map $f$ and $X_t=g+\sqrt t\,\sigma Z$, Gaussian variance differentiation
\citep{voigtlaender2021price} gives the exact identity
\begin{equation}
 \frac{d}{dt}\E[f(X_t)]=\frac{\sigma^2}{2}\E[f''(X_t)],\qquad
 \E[f(X_R)]-\E[f(X_1)]
 =\frac{\sigma^2}{2}\int_1^R\E[f''(X_t)]\,dt.
 \label{rw:eq:heat}
\end{equation}
Differentiation under the integral is justified, for example, by bounded first
and second derivatives and Gaussian integrability; the Adam map below satisfies
the required conditions when the incoming second moment is positive.

\subsubsection{Finite-noise mean bound and second-moment memory}
\label{rw:app:beta-mean}

\begin{proof}[Proof of Proposition~\ref{rw:prop:beta-mean}]
For this proof only, abbreviate
$p=\beta_1m$, $k=1-\beta_1$, $a_0=\beta_2v>0$, and $b_0=1-\beta_2>0$.
Write
\[
 f(x)=p\,h(x)+k\,\ell(x),\qquad
 h(x)=(a_0+b_0x^2)^{-1/2},\quad
 \ell(x)=x(a_0+b_0x^2)^{-1/2}.
\]
Direct differentiation gives
\begin{equation}
 h''(x)=\frac{b_0(2b_0x^2-a_0)}{(a_0+b_0x^2)^{5/2}},\qquad
 \ell'''(x)=\frac{-3a_0b_0(a_0-4b_0x^2)}{(a_0+b_0x^2)^{7/2}}.
 \label{rw:eq:derivative-bounds-form}
\end{equation}
For $y\geq0$, the suprema of
$|2y-1|/(1+y)^{5/2}$ and $|1-4y|/(1+y)^{7/2}$ are both one.
The first follows by differentiation on either side of $y=1/2$, and the
second on either side of $y=1/4$; their interior positive extrema have value
less than one. Substituting $y=b_0x^2/a_0$ gives
\[
 \|h''\|_\infty\leq b_0/a_0^{3/2},\qquad
 \|\ell'''\|_\infty\leq3b_0/a_0^{3/2}.
\]
Since $\ell''$ is odd and $Z$ is symmetric,
$\E[\ell''(\sqrt t\sigma Z)]=0$. The mean-value bound therefore implies
\[
 \big|\E[\ell''(g+\sqrt t\sigma Z)]\big|
 \leq 3b_0|g|/a_0^{3/2}.
\]
Thus, uniformly in $t>0$,
\[
 |\E[f''(g+\sqrt t\sigma Z)]|
 \leq\frac{b_0}{a_0^{3/2}}\bigl(|p|+3k|g|\bigr).
\]
Integrating \eqref{rw:eq:heat} over $1\leq t\leq R$ gives
\eqref{rw:eq:beta-mean-bound}. This proof does not replace the random denominator
by its expectation and does not require a small-noise limit.
\end{proof}

\paragraph{Scaling rule}
A local example or token is downweighted by $\beta_2$ per local step. Keeping
its decay per fixed number of local examples unchanged gives
$\log\beta_2(R)/(B/R)=\log\beta_2(R_0)/(B/R_0)$, proving
\eqref{rw:eq:beta-rule}. Write $c_0=-R_0\log\beta_2(R_0)>0$. Then
$1-\beta_2(R)=1-e^{-c_0/R}\leq c_0/R$, so
$(R-1)(1-\beta_2(R))\leq c_0$. This controls the upper bound with fixed
incoming state; it neither forces the bias to zero nor proves that its actual
magnitude decreases monotonically with $\beta_2$.

%% file: appendices/basis_theory.tex
\subsection{Basis and scale errors}
\label{g:app:basis}

SOAP whitening requires both a basis that reduces correlations and coordinate
scales that normalize the remaining variances. We first separate these two
sources of error in the covariance of the returned update, then show how they
affect covariance matching and the stochastic term in
Equation~\eqref{rw:eq:main-decomposition}.

\paragraph{Conditioning and scope.}
\textbf{The covariance calculations below concern centered proposals with
bases and denominators fixed before the proposal is drawn.} They describe an
idealized one-step optimizer core, omitting a common learning-rate factor and
the numerical stabilizer. The ideal denominators estimate the centered
variance of this proposal. In practical SOAP, the denominator instead uses
gradient second moments, whereas the numerator contains momentum; refreshing
either state from the same gradient also couples the map to its input. The
identities below therefore describe the frozen-state model, not the full
fresh-gradient Adam or SOAP step.

\subsubsection{Scale errors and correlations left in the estimated basis}
\label{g:app:scale-basis}

Let worker $r$ have a centered matrix proposal
$\bm M_r\in\mathbb R^{m\times n}$, and set
$\bm x_r=\vecop(\bm M_r)$ and
$\bm\Sigma_r=\Cov(\bm x_r)\succ\bm 0$. Let
$\widehat{\bm Q}_{L,r}$ and $\widehat{\bm Q}_{R,r}$ be the fixed estimated
left and right bases. They are orthogonal matrices. Rotating into these
bases gives the matrix proposal on which normalization or sign acts:
\begin{equation}
 \bm Z_r=\widehat{\bm Q}_{L,r}^{\top}\bm M_r
                  \widehat{\bm Q}_{R,r},\qquad
 \bm Q_r=\widehat{\bm Q}_{R,r}\otimes\widehat{\bm Q}_{L,r},\qquad
 \vecop(\bm Z_r)=\bm Q_r^{\top}\bm x_r.
 \label{g:eq:matrix-rotation}
\end{equation}
No Kronecker assumption is imposed on $\bm\Sigma_r$ here: the Kronecker
matrix $\bm Q_r$ represents the two-sided rotation, not a factorization of
the full covariance.

To distinguish scale error from an imperfect basis, let $\bm\Gamma_r$ be the
true covariance in the estimated basis, $\bm D_r$ its diagonal of coordinate
variances, and $\bm{\mathcal C}_r$ the correlation matrix remaining after
exact variance normalization:
\begin{equation}
 \bm\Gamma_r=\bm Q_r^{\top}\bm\Sigma_r\bm Q_r,\qquad
 \bm D_r=\operatorname{Diag}(\bm v_r),\quad
 v_{rj}=(\bm\Gamma_r)_{jj},\qquad
 \bm{\mathcal C}_r=\bm D_r^{-1/2}\bm\Gamma_r\bm D_r^{-1/2}.
 \label{g:eq:rotated-correlation}
\end{equation}
Thus $\bm{\mathcal C}_r$ has unit diagonal; its off-diagonal entries measure
correlations that the basis has not removed. The actual normalization uses
frozen variance estimates $\widehat v_{rj}$. Write their relative errors as
$e_{rj}$ and collect the resulting inverse-scale gains in $\bm A_r$:
\begin{equation}
 \widehat v_{rj}=v_{rj}(1+e_{rj})>0,\qquad
 \widehat{\bm D}_r=\operatorname{Diag}(\widehat{\bm v}_r),\qquad
 \bm A_r=\operatorname{Diag}\big((1+e_{rj})^{-1/2}\big).
 \label{g:eq:relative-scale-error}
\end{equation}
If $\widehat{\bm V}_r$ is the $m\times n$ reshape of $\widehat{\bm v}_r$,
the SOAP core divides by these estimated scales, whereas SPlus takes signs.
Both then rotate the result back to the parameter coordinates:
\begin{align}
 \bm U_r^{\rm SOAP}
 &=\widehat{\bm Q}_{L,r}
       \left(\bm Z_r\oslash\sqrt{\widehat{\bm V}_r}\right)
       \widehat{\bm Q}_{R,r}^{\top},\nonumber\\
 \vecop(\bm U_r^{\rm SOAP})
 &=\bm Q_r\widehat{\bm D}_r^{-1/2}\bm Q_r^{\top}\bm x_r,
 \label{g:eq:soap-matrix-map}\\
 \bm U_r^{\rm SPlus}
 &=\widehat{\bm Q}_{L,r}\sign(\bm Z_r)
       \widehat{\bm Q}_{R,r}^{\top},\nonumber\\
 \vecop(\bm U_r^{\rm SPlus})
 &=\bm Q_r\sign(\bm Q_r^{\top}\bm x_r).
 \label{g:eq:suds-matrix-map}
\end{align}
Division, square roots, and signs in these matrix expressions act entrywise.

\begin{proposition}[Separate effects of the basis and coordinate scales]
\label{g:prop:scale-basis}
Under the preceding frozen-state assumptions, define
$\bm C_r^{\rm SOAP}=\Cov(\vecop(\bm U_r^{\rm SOAP}))$.
For every centered proposal with finite covariance,
\begin{equation}
 \bm C_r^{\rm SOAP}
 =\bm Q_r\bm A_r\bm{\mathcal C}_r\bm A_r\bm Q_r^{\top},
 \label{g:eq:soap-covariance}
\end{equation}
and the whitening error decomposes exactly as
\begin{equation}
 \|\bm C_r^{\rm SOAP}-\bm I\|_F^2
 =\underbrace{\sum_j\left(\frac{e_{rj}}{1+e_{rj}}\right)^2}
                  _{\text{coordinate-scale error}}
  +\underbrace{\sum_{i\ne j}
       \frac{\mathcal C_{r,ij}^2}{(1+e_{ri})(1+e_{rj})}}
                  _{\text{residual correlations after normalization}}.
 \label{g:eq:soap-whitening-error}
\end{equation}
If additionally $\bm x_r$ is centered Gaussian, then
\begin{align}
 \bm C_r^{\rm SPlus}
 &=\bm Q_r\left[\frac2\pi\arcsin(\bm{\mathcal C}_r)\right]
                   \bm Q_r^{\top},
 \label{g:eq:suds-covariance}\\
 \|\bm C_r^{\rm SPlus}-\bm I\|_F^2
 &=\sum_{i\ne j}
       \left[\frac2\pi\arcsin(\mathcal C_{r,ij})\right]^2
 \leq\|\bm{\mathcal C}_r-\bm I\|_F^2.
 \label{g:eq:suds-whitening-error}
\end{align}
Here $\arcsin$ acts entrywise. The same sign identities hold for the
positive scale-mixture model
$\bm x_r=\sqrt{\tau_r}\bm y_r$, where
$\bm y_r\sim\mathcal N(\bm0,\bm\Sigma_r)$,
$\tau_r>0$ is independent of $\bm y_r$, and $\E[\tau_r]=1$.
\end{proposition}

\begin{proof}
Vectorization gives Equations~\eqref{g:eq:soap-matrix-map} and
\eqref{g:eq:suds-matrix-map}. Since the maps and denominators are frozen,
the SOAP covariance is the linear covariance sandwich. Substitute
$\bm\Gamma_r=\bm D_r^{1/2}\bm{\mathcal C}_r\bm D_r^{1/2}$ and
$\widehat{\bm D}_r^{-1/2}\bm D_r^{1/2}=\bm A_r$ to obtain
Equation~\eqref{g:eq:soap-covariance}. Orthogonal conjugation preserves
Frobenius norm. In the rotated frame, the diagonal covariance entries are
$(1+e_{rj})^{-1}$ and the off-diagonal entries are
$\mathcal C_{r,ij}/\sqrt{(1+e_{ri})(1+e_{rj})}$.
Summing their squared differences from the identity proves
Equation~\eqref{g:eq:soap-whitening-error}.

For the sign map, each centered nondegenerate Gaussian coordinate has an
equiprobable sign and hence mean zero and variance one. Two standardized
Gaussian coordinates with correlation $\rho$ have sign disagreement
probability $\arccos(\rho)/\pi$: represent them as projections of an
isotropic two-dimensional Gaussian onto two unit directions with angle
$\arccos(\rho)$ and use its uniform angular distribution. Their sign
product therefore has expectation
$1-2\arccos(\rho)/\pi=(2/\pi)\arcsin(\rho)$. This is the Gaussian
arcsine identity used in one-bit covariance estimation
\citep{dirksen2022onebit}, and proves Equation~\eqref{g:eq:suds-covariance}.
For $0\leq t\leq1$, convexity of $\arcsin(t)$ places it below the chord
joining its endpoint values, so $(2/\pi)\arcsin(t)\leq t$.
Oddness extends the absolute-value inequality to $[-1,1]$. Squaring the
off-diagonal entries proves Equation~\eqref{g:eq:suds-whitening-error}.
Finally, multiplication by the positive scalar $\sqrt{\tau_r}$ leaves
every sign unchanged; independence and $\E[\tau_r]=1$ preserve the stated
input covariance.
\end{proof}

\textbf{An accurate basis does not remove SOAP's scale-estimation error.}
The first term remains when the basis decorrelates the proposal but the
local scales are inaccurate; the second remains when exact scales are used
in a basis that leaves correlations. SPlus removes the estimated-scale term
in this centered model, while retaining dependence on those correlations.
In particular, if
$\bm{\mathcal C}_r=\bm I$, SPlus has identity covariance for every worker,
including workers using different bases and different input variances.

For an explicit bound, if $\max_j|e_{rj}|\leq\epsilon_r<1$, then
\begin{equation}
 \|\bm C_r^{\rm SOAP}-\bm I\|_F^2
 \leq\frac{\|\bm e_r\|_2^2+
                  \|\bm{\mathcal C}_r-\bm I\|_F^2}
                {(1-\epsilon_r)^2}.
 \label{g:eq:relative-error-bound}
\end{equation}
This follows by bounding each denominator in
Equation~\eqref{g:eq:soap-whitening-error} from below. Underestimating a
variance can therefore amplify the whitening error through its reciprocal.
Rotation back, by contrast, preserves the error in either operator or
Frobenius norm:
\begin{equation}
 \|\bm Q_r\bm B\bm Q_r^{\top}-\bm I\|
 =\|\bm B-\bm I\|.
 \label{g:eq:rotation-invariance}
\end{equation}
Orthogonal multiplication preserves singular values, proving
Equation~\eqref{g:eq:rotation-invariance}. Rotation back can change the
directions of residual anisotropy, but cannot increase this norm of its error.

%% file: appendices/training_details.tex
\section{Training setup and hyperparameters}
\label{sec:hyperparameters}

\paragraph{Models and data.}
The experiments use decoder-only LLaMA models trained on
DCLM~\citep{li2024dclm}, with sequence length 2048 and evaluation on a
held-out 100M-token validation split. The model configuration files give
the architecture below. Global batch size is held fixed within each model
scale, so the local batch is the global batch divided by $R$.

\begin{table}[H]
\centering\small
\begin{tabular}{lrr}
\toprule
Setting & 178M & 512M\\
\midrule
Transformer layers & 9 & 12\\
Hidden width & 1024 & 1536\\
Attention heads / key-value heads & $8/8$ & $12/12$\\
Feed-forward width & 2688 & 5440\\
Activation & SwiGLU & SwiGLU\\
Sequence length & 2048 & 2048\\
Global tokens per inner step & 524,288 & 2,097,152\\
Global sequences per inner step & 256 & 1024\\
Training tokens (approximately) & 3.55B & 10.25B\\
\bottomrule
\end{tabular}
\caption{Architecture and training scales. The global batch is measured
per inner step, not per communication round.}
\label{tab:training-setup}
\end{table}

\paragraph{Aggregation and comparisons.}
Unless stated otherwise, replicas take $H=30$ local steps and use a
Nesterov SGD outer optimizer. The $H=1$ controls use uniform averaging
without outer momentum. Percentage changes in validation loss use each
optimizer's own DP reference and are computed before rounding.
The covariance, mean/noise, and curvature probes use the 178M
outer-step-75 checkpoints; their replay settings and sample counts are
specified in Appendix~\ref{app:g-mechanism-details}.

\paragraph{Hyperparameter sweep.} We tuned the learning rate for each optimizer in the DP setting and the inner and outer learning rates and outer momentum for local training. For each optimizer we tuned R=8,64,96 each with at least 10 grid points extrapolating the HP for R=16 and 32.

%% file: appendices/analysis_protocols.tex
\section{Analysis protocols and estimators}
\label{app:g-mechanism-details}

All three analyses use the 178M outer-step-75 checkpoints. The covariance
and curvature probes use six matrices (attention query and MLP down-projection
at layers 0, 4, and 8); the DP-referenced mean/noise comparison covers all
63 hidden attention and MLP matrices. Each optimizer uses its own trained
checkpoint. The protocols below distinguish native saved states from the
common-state replay used for comparison with DP.

\subsection{Conditional covariance replay}
\label{app:fixed-covariance-estimator}

This appendix spells out the covariance measurement used in
Table~\ref{tab:fixed-covariance}. The
probe estimates the minibatch-induced distribution of an optimizer's local
update conditional on one trained checkpoint. It does not treat successive
updates from the training trajectory as independent covariance samples.

\subsubsection{Conditional replay protocol}

At a communication boundary, all replicas share the synchronized parameters
$\bm{W}$, while replica $r$ has its own saved optimizer state $s_r$. Table~\ref{tab:fixed-covariance} uses all 64 saved replica states. For each replica,
we run $2n$ rollouts: $n=16$. Before every rollout, we restore the same pair
$(\bm{W},s_r)$, then consume a new segment of the replica's shuffled data stream.
Within a rollout, the parameters and optimizer state evolve normally. We record
the parameter displacement after $h\in\{1,30\}$ local steps,
\begin{equation}
  \bm{u}_{r,j}^{(h)}
  =\mathop{\operatorname{concat}}_{\ell\in\mathcal L}
    \operatorname{vec}\!\left(\bm{W}_\ell-\bm{W}_{r,j,\ell}^{(h)}\right),
  \qquad j=1,\ldots,2n,
  \label{eq:probe-raw-update}
\end{equation}
where $\mathcal L$ contains attention \texttt{q\_proj} and MLP
\texttt{down\_proj} at layers 0, 4, and 8. For the 178M model, this concatenates
$d=11{,}403{,}264$ coordinates. The $h=1$ and $h=30$ endpoints come from the
same rollout and therefore have coupled finite-sample errors.

The per-replica batch is four sequences for $R=64$,
giving a global batch of 256 sequences. Each replay uses the optimizer
configuration saved by the corresponding training run, including its replica-local
momentum and second moments and its saved replica-specific or shared basis.

For fixed $r$ and $h$, write
\begin{equation}
  \bm{\mu}_{r,h}=\E[\bm{u}_{r,j}^{(h)}],
  \qquad \bm{C}_{r,h}=\Cov(\bm{u}_{r,j}^{(h)}).
\end{equation}
Two consecutive rollouts form one paired difference,
\begin{equation}
  \bm{z}_{r,i}^{(h)}
  =\frac{\bm{u}_{r,2i-1}^{(h)}-\bm{u}_{r,2i}^{(h)}}{\sqrt 2},
  \qquad i=1,\ldots,n.
  \label{eq:probe-paired-difference}
\end{equation}
If the two rollouts are independent draws from the same conditional
distribution, then
\begin{equation}
  \E[\bm{z}_{r,i}^{(h)}]=\bm{0},
  \qquad
  \Cov(\bm{z}_{r,i}^{(h)})
  =\tfrac12(\bm{C}_{r,h}+\bm{C}_{r,h})=\bm{C}_{r,h}.
  \label{eq:probe-pair-covariance}
\end{equation}
The subtraction therefore removes the unknown mean, while $1/\sqrt2$ restores
the original covariance scale.

\subsubsection{Trace estimators from full-coordinate inner products}

Fix a horizon and suppress $h$ in the notation. Define the Gram entries
\begin{equation}
  G_{ri,sj}=\bm{z}_{r,i}^{\top}\bm{z}_{s,j}.
\end{equation}
The code computes these inner products over every selected coordinate; it does
not use a random projection. It nevertheless estimates population moments
from $n$ paired differences per replica.

For each replica, define
\begin{align}
  t_r
  &=\frac1n\sum_{i=1}^n G_{ri,ri},
  \label{eq:probe-trace-estimator}\\
  b_r
  &=\frac{1}{n(n-1)}\sum_{i\neq j}
    G_{ri,ri}G_{rj,rj},
  \label{eq:probe-trace-square-estimator}\\
  q_{rr}
  &=\frac{1}{n(n-1)}\sum_{i\neq j}G_{ri,rj}^2.
  \label{eq:probe-c2-estimator}
\end{align}
Under the IID-rollout idealization, for independent paired differences,
\begin{equation}
  \E[t_r]=\tr(\bm{C}_r),\qquad
  \E[b_r]=\tr(\bm{C}_r)^2,\qquad
  \E[q_{rr}]=\tr(\bm{C}_r^2).
\end{equation}
The last identity follows from
\begin{equation}
  \E[(\bm{z}^\top \bm{z}')^2]
  =\tr\!\left(\E[\bm{z}\bm{z}^\top]
    \E[\bm{z}'\bm{z}'^\top]\right)
  =\tr(\bm{C}^2)
\end{equation}
for independent $\bm{z}$ and $\bm{z}'$. Using $i\neq j$ is important: including the
diagonal would add fourth-moment self-products and bias the estimate upward.

For different replicas $r\neq s$, the cross term
\begin{equation}
  q_{rs}=\frac1{n^2}\sum_{i,j}G_{ri,sj}^2
  \label{eq:probe-cross-c2-estimator}
\end{equation}
estimates $\tr(\bm{C}_r\bm{C}_s)$ under the same IID and conditional
replica-independence idealization. The quantity is the scale-dependent
Frobenius overlap between two replicas' covariance matrices, not a
cross-covariance between their stochastic updates.

\subsubsection{Scaled sum of local covariances}

The probe constructs the scaled sum
\begin{equation}
  \widetilde{\bm{C}}=\frac1{R^2}\sum_{r=1}^R \bm{C}_r.
  \label{eq:probe-average-covariance}
\end{equation}
For the actual replica average $\bar{\bm u}=R^{-1}\sum_r\bm u_r$,
this equals $\Cov(\bar{\bm u})$ if the cross-replica noise covariances
vanish conditional on the starting parameters and optimizer states.
The local-replay probe does not measure those cross-covariances.
To estimate the trace moments of $\widetilde{\bm C}$, it uses
\begin{align}
  \widehat T
  &=\frac1{R^2}\sum_r t_r,
  &\E[\widehat T]&=\tr(\widetilde{\bm{C}}),
  \label{eq:probe-average-trace}\\
  \widehat S
  &=\frac1{R^4}\sum_{r,s}q_{rs},
  &\E[\widehat S]&=\tr(\widetilde{\bm{C}}^2),
  \label{eq:probe-average-c2}\\
  \widehat B
  &=\frac1{R^4}\left(\sum_r b_r+
    \sum_{r\neq s}t_rt_s\right),
  &\E[\widehat B]&=\tr(\widetilde{\bm{C}})^2.
  \label{eq:probe-average-trace-square}
\end{align}
The separate estimator $\widehat B$ avoids squaring the noisy estimate
$\widehat T$. The reported anisotropy is
\begin{equation}
  \widehat a
  =\sqrt{\max\!\left\{0,
    \frac{d\widehat S}{\widehat B}-1\right\}}.
  \label{eq:probe-reported-covariance-metrics}
\end{equation}
Under the IID idealization, $\widehat T$, $\widehat S$, and $\widehat B$ are
built from unbiased moment estimators, but their ratio and square root are not
themselves unbiased. The clipping only handles a negative raw value caused by
finite-sample error; it can therefore return zero even when the population
covariance is anisotropic.

If $\lambda_1,\ldots,\lambda_d$ are the eigenvalues of $\widetilde{\bm{C}}$, the
population quantity is
\begin{equation}
  a^2=\frac{d\sum_i\lambda_i^2}{(\sum_i\lambda_i)^2}-1,
  \qquad
  \frac{r_{\mathrm{eff}}}{d}=\frac{1}{1+a^2},
  \label{eq:probe-effective-rank}
\end{equation}
where $r_{\mathrm{eff}}=(\sum_i\lambda_i)^2/\sum_i\lambda_i^2$ is the
participation-ratio effective rank. Thus $a$ measures covariance shape
independently of scale; it is undefined when the covariance is identically
zero. For example,
$\operatorname{diag}(5,5)$ has $a=0$, whereas
$\operatorname{diag}(9,1)$ has $a=0.8$. Multiplying either matrix
by a positive scalar leaves $a$ unchanged.

\subsubsection{Scope and interpretation}

The estimator is conditional on one checkpoint and one set of replica-local
states. Each method is evaluated on its own trained trajectory, so a
cross-method difference can reflect both the update rule and the state reached
by training; it is not a same-state causal ablation. Concatenating six matrices
also mixes within-matrix anisotropy, scale imbalance between matrices, and
cross-matrix covariance. The input streams are disjoint segments of one
shuffled replica-local permutation, rather than independent samples drawn with
replacement. The unbiasedness equalities above are therefore exact only under
the stated IID idealization. Moreover, the $n(n-1)$ and $n^2$ sums reuse the
same $n$ vectors; their terms are not independent observations. Finally,
interpreting Equation~\eqref{eq:probe-average-covariance} as the covariance
of the actual average requires vanishing conditional cross-replica
noise covariances.

Lower $a$ indicates a more isotropic $\widetilde{\bm C}$,
independently of total noise magnitude.
For SOAP, the saved position relative to a basis-refresh step can also
affect the measured covariance: the fresh gradient may enter both the basis
and the proposal. At this single checkpoint, the probe characterizes update
geometry but does not identify its causal contribution to final loss.

\subsection{DP-referenced directional mean and noise}
\label{app:g-directional-estimator}

Each optimizer is evaluated at its own round-75 checkpoint. At the
start of every repetition, the checkpoint weights and replica 0's
saved optimizer state are copied to every replica and to the DP
reference. ``Common state'' means identical states within that replay,
not an averaged optimizer state or a state shared across optimizers.
Each replica consumes a persistent, shuffled data stream without
reusing sequences across repetitions. In each step DP sees the union
of the same local batches, evaluated at its own iterate. This preserves
the pairing between local and DP data without reusing examples between
the direction-fitting and evaluation samples.

Let $j$ index one of the 63 hidden matrices, vectorized in parameter
coordinates. Write $\bm x_{j,k}$ for its averaged local displacement
and $\bm y_{j,k}$ for its DP displacement in repetition $k$.
On the 16 fitting repetitions at $R=8$, define
\[
 \widehat\alpha_j
 =\frac{\langle\overline{\bm x}_{j,\rm fit},
                         \overline{\bm y}_{j,\rm fit}\rangle}
        {\norm{\overline{\bm x}_{j,\rm fit}}^2},
\]
and take $\widehat{\bm q}_j$ to be the leading eigenvector of the
sample covariance of $\widehat\alpha_j\bm x_{j,k}$ on those same
fitting repetitions. The reference mean is
$\bm d_j=\E[\bm y_{j,k}]$, and the mean discrepancy is
$\bm b_j=\widehat\alpha_j\E[\bm x_{j,k}]-\bm d_j$, conditional on
the fitting sample. The decomposition in
Equation~\eqref{rw:eq:main-decomposition} therefore applies to each
matrix with these fitted quantities held fixed.

We use the remaining $K=24$ independent repetitions for evaluation.
Set $\bm z_{j,k}=\widehat\alpha_j\bm x_{j,k}-\bm y_{j,k}$,
$t_{j,k}=\widehat{\bm q}_j^\top\bm z_{j,k}$, and
$p_{j,k}=\widehat\alpha_j\widehat{\bm q}_j^\top\bm x_{j,k}$.
With $s^2$ denoting the unbiased sample variance on these repetitions,
the unnormalized estimates are
\[
 \widehat B_j=\overline t_j^{\,2}-\frac{s^2(t_{j,k})}{K},
 \qquad
 \widehat N_j=s^2(p_{j,k}).
\]
Conditional on the fit and independent evaluation repetitions,
$\widehat B_j$ is unbiased for
$(\widehat{\bm q}_j^\top\bm b_j)^2$ and $\widehat N_j$ estimates
$\widehat\alpha_j^2\widehat{\bm q}_j^\top\bm C_j\widehat{\bm q}_j$,
where $\bm C_j=\Cov(\bm x_{j,k})$.
Crucially, the correction uses the variance of the \emph{paired}
difference $\bm z_{j,k}$; it retains the local--DP cross-covariance.
The stochastic term is the variance of a single aggregated update,
not that variance divided by the number of evaluation repetitions.

For each matrix, the first two numerical columns in Table~\ref{tab:g-mean-noise}
are $\widehat B_j/\norm{\widehat{\bm d}_j}^2$ and
$\widehat N_j/\norm{\widehat{\bm d}_j}^2$, with
$\widehat{\bm d}_j=K^{-1}\sum_k\bm y_{j,k}$, expressed in units of
$10^{-6}$. The third column is $100\widehat B_j/\widehat N_j$.
Every column is averaged with equal weight over the 63 matrices.
In particular, the reported average ratio is not the ratio of the
two reported averages. Bias-corrected mismatch estimates can be
negative and are not clipped before averaging. The estimated DP
normalizer and the ratios introduce additional finite-sample
uncertainty; the normalized entries are not claimed to be unbiased.
The matrices share data and are not independent experimental repeats.

\paragraph{Trace-based comparison.}
The corresponding whole-matrix estimates are
\[
 \widehat B_j^{\mathrm{tr}}
 =\norm{\overline{\bm z}_j}^{2}
 -\frac{1}{K(K-1)}\sum_k
       \norm{\bm z_{j,k}-\overline{\bm z}_j}^{2},
 \qquad
 \widehat N_j^{\mathrm{tr}}
 =\frac{\widehat\alpha_j^2}{K-1}\sum_k
       \norm{\bm x_{j,k}-\overline{\bm x}_j}^{2}.
\]
They estimate $\norm{\bm b_j}^{2}$ and
$\widehat\alpha_j^2\tr(\bm C_j)$, respectively.
The same DP-mean normalizer and equal per-matrix averaging apply;
the mean/noise ratio is computed per matrix before averaging.
At $R=8,H=30$, these average ratios are $6.2\%$ for SPlus,
$2.8\%$ for SOAP, $10\%$ for Muon, and $6.7\%$ for AdamW.

\subsection{Curvature along fitted high-noise directions}
\label{app:g-curvature-estimator}

At each method's own checkpoint, restore the saved parameters and optimizer
state before every one-step replay. Form 128 paired update differences
$\bm z_i=(\bm u_i-\bm u_i')/\sqrt2$, using fresh local batches of
$b=256/R$ sequences. Under the independent-replay model these differences
are centered and have the local-update covariance. For Table~\ref{tab:g-curvature}, use the first 64 of the 128 paired
samples to select
\[
 \widehat{\bm q}=\underset{\norm{\bm q}=1}{\arg\max}\;
 \frac1{64}\sum_{i=1}^{64}(\bm q^\top\bm z_i)^2.
\]
Hessian--vector products on 64 held-out validation sequences evaluate
curvature at the same checkpoint parameters. Draw 16 independent
standard Gaussian vectors in the six-matrix subspace and normalize
each to unit length, giving $\bm v_1,\ldots,\bm v_{16}$. Report
\[
 \kappa(\widehat{\bm q})=
 \frac{\widehat{\bm q}^{\top}\bm{\mathcal H}\widehat{\bm q}}
 {\frac1{16}\sum_{j=1}^{16}\bm v_j^{\top}\bm{\mathcal H}\bm v_j}.
\]
The numerator measures curvature along a direction selected for high
noise variance, not maximum curvature. The denominator estimates
$\tr(\bm{\mathcal H})/d$ in this subspace. Each method uses its own
trained checkpoint, so the Hessian and reference curvature are
method-specific. Ranges span the two estimates from replicas 0 and 1;
a single displayed value means both endpoints round to it. These are
one-step local-update measurements, not $H=30$ aggregate measurements.
The ratio removes the noise magnitude and does not estimate
$\tr(\bm{\mathcal H}\bm C)$.